%% file: CorePeriphery.tex
\documentclass[12pt]{article}
\usepackage{amsmath}
\usepackage{amsthm}
\usepackage{graphicx}
\usepackage{subcaption}
\usepackage{enumerate}
\usepackage{natbib}
\usepackage{hyperref}       
\usepackage{cleveref}
\usepackage{url} 
\usepackage{amsfonts}       
\usepackage{algorithm}
\usepackage{algorithmic}
\usepackage{makecell}
\usepackage{multirow}
\usepackage{bm}
\usepackage{color}

\newcommand{\rank}{\mathrm{rank}}
\newcommand{\diag}{\mathrm{diag}}

\newcommand{\ccal}{\mathcal{C}}
\newcommand{\pcal}{\mathcal{P}}

\usepackage{natbib}
\newtheorem{definition}{Definition}
\newtheorem{theorem}{Theorem}

\newtheorem{lemma}{Lemma}
\newtheorem{assumption}{Assumption}
\newtheorem{corollary}{Corollary}

\begin{document}

\def\spacingset#1{\renewcommand{\baselinestretch}%
{#1}\small\normalsize} \spacingset{1}

\newcommand{\norm}[1]{\left\lVert#1\right\rVert}
\newcommand{\abs}[1]{\left|#1\right|}


\title{\bf Informative core identification in complex networks}

  \author{\\
  Ruizhong Miao and Tianxi Li\\
 \\
  Department of Statistics\\
    University of Virginia}
 \maketitle

\bigskip
\begin{abstract}
In network analysis, the core structure of modeling interest is usually hidden in a larger network in which most structures are not informative. The noise and bias introduced by the non-informative component in networks can obscure the salient structure and limit many network modeling procedures' effectiveness. This paper introduces a novel core-periphery model for the non-informative periphery structure of networks without imposing a specific form for the informative core structure. We propose spectral algorithms for core identification as a data preprocessing step for general downstream network analysis tasks based on the model. The algorithm enjoys a strong theoretical guarantee of accuracy and is scalable for large networks. We evaluate the proposed method by extensive simulation studies demonstrating various advantages over many traditional core-periphery methods. The method is applied to extract the informative core structure from a citation network and give more informative results in the downstream hierarchical community detection.
\end{abstract}


\newpage
\spacingset{1.5} 

\input{intro}

\input{model}

\input{theory}

\input{simulation}

\input{data}

\section{Discussion}\label{sec:conc}

We have proposed a core-periphery model for extracting informative structures from networks and proposed two efficient algorithms for core identification under the model. Our model does not assume a specific form for the core component, so it can be used for preprocessing for downstream network modeling in general. The proposed algorithms have theoretical guarantees of correctly identifying the core component under mild conditions. The strong consistency property is advantageous for our model since conditioning on the core extract success, any downstream network theoretical analyses will remain valid on the core part.

There are several possible extensions to pursue following the proposed framework. For example, what are the other generally uninteresting structures in network model cases, and would they be incorporated in the same framework? Another interesting question is how to generalize the current framework to more complicated data structures for network modeling settings such as the multiplex networks and dynamic networks. Such extensions may require delicate definitions of uninteresting structures in the new scenarios and potentially new model fitting tools.

\section*{Acknowledgements}

This work was supported in part by the NSF grant DMS-2015298 and the Quantitative Collaborative grant from the College of Arts \& Sciences at the University of Virginia. The authors want to thank Lihua Lei for his helpful suggestions.

\bibliography{CorePeriphery}

\input{appendix}

\end{document}

%% file: intro.tex
\section{Introduction}\label{sec:intro}

Network data, representing interactions and relationships between units, have become ubiquitous with the rapid development of science and technology. Analyzing such complex and structurally novel data has resulted in a rich body of new ideas and tools in physics, mathematics, statistics, computer science, and social sciences. In particular, given that complex network structures are typically noisy and complicated, treating the network as a random instantiation of a probabilistic model has been widely used to learn the structural properties while ignoring unnecessary noisy details. This approach can be traced back as early as the work of  \citet{erdos1959random}. Later work of \cite{aldous1981representations,hoover1979relations} further set up foundations and frameworks for more flexible random network modeling. More recently, significant progress has been achieved to make network analysis more computationally efficient, scientifically interpretable with theoretical guarantees \citep{albert2002statistical, hoff2002latent, bickel2009nonparametric, zhao2012consistency, newman2016equivalence, gao2017achieving, athreya2017statistical, mukherjee2018mean}. 

Though these methods have been used to solve many significant problems in different fields, empirically, they sometimes fail to learn structural information effectively. This is because most network models assume a particular type of structure of interest. However, one issue that complicates matters in practice is the scarcity of interesting or informative structures in large-scale networks. In other words, the presumed structure of interest may only be valid for a subnetwork, while the rest of the network may be noninformative. For example, it is observed by \citet{ugander2013subgraph} that the first a few moments in 100 Facebook subnetworks are very similar to the Erd\"{o}s-Renyi model. Moreover, \citet{gao2017testing2} tested these networks, observing that most of them show no evident difference from purely random connections and admit no interesting structure. For another example, preprocessing was applied in \citet{wang2016discussion,li2020network,li2020hierarchical} to remove a subset of nodes before the community detection algorithms were applied. Such preprocessing is reported as a crucial step for successful community analysis. In these analyses,  the $k$-core pruning algorithm \citep{seidman1983network} was used to remove some low-degree nodes, and the networks under study were assumed to have a core-periphery structure, a natural framework for our current problem. For an illustration, in \Cref{fig:CP_Example_11}, we plot the top eigenvalues of a random network generated by the model in \Cref{fig:CP_const} (details can be found in Section~\ref{sec:proposedModel}), and periphery nodes follow our proposed model. The network model has rank 3. Hence, when the signal-to-noise ratio is manageable, we would observe a large eigengap between the 3rd and the 4th eigenvalues. As we increase the number of periphery nodes, however, the eigengap vanishes, and the model looks like rank 1, obscuring the informative structure. An effective preprocessing method should correctly identify the core and filter out the peripheries.

The core-periphery structure has been studied in network literature for long. For example, \citet{borgatti2000models} define the structure as a special case of the stochastic block model \citep{holland1983stochastic}. This definition of core-periphery is used by \citet{zhang2015identification, priebe2019two} as well as a related problem called ``planted clique problem" \citep{alon1998finding,dekel2014finding}. Under this definition, the network core is a densely connected Erd\"{o}s-R\'{e}nyi network, which is too restrictive to be interesting settings for any downstream analysis. Meanwhile, this definition heavily relies on the density gap between the core and the periphery \citep{zhang2015identification, kojaku2018core} which may not be true in many applications. \cite{naik2019sparse} recently propose another core-periphery model. The core structure is more general than the Erd\"{o}s-R\'{e}nyi but still follows a restrictive parametric form. Moreover, the model can only generate networks with node degrees at least as dense as the square root of the network size, which is too dense to model most real-world networks. On the other hand, algorithm-based methods \citep{lee2014density, della2013profiling, barucca2016centrality,cucuringu2016detection,rombach2017core} typically assign a ``coreness" score to each node based on certain topological assumptions. This class of methods is not well-understood in their statistical properties. Another related research problem is the submatrix localization problem \citep{butucea2015sharp,deshpande2015improved,hajek2017information,cai2017computational}. The objective is to find $K$ densely connected subgraphs planted in a large Erd\"{o}s-R\'{e}nyi graph in this type of problem, and the $K$ subgraphs are usually assumed to be Erd\"{o}s-R\'{e}nyi graphs, which is again too restrictive in practice.

This paper aims to bridge the gap between the theoretically predicted effectiveness of network modeling and the empirical expectation in data analysis by proposing a principled and computationally efficient preprocessing method of extracting the informative structure from the non-informative background noise. We introduce a core-periphery model for informative and non-informative structures. The novelty of our model comes in two folds. Firstly, unlike traditional definitions, our distinction between the core and periphery components is whether the component has informative connection patterns. Secondly, our model does not assume a specific model for the core component. These two substantive distinctions highlight the advantages of our method. Since we do not constrain our core structure to a specific network model, our framework admits the generality needed as a preprocessing step for any downstream network analysis. Meanwhile, our core-periphery definition emphasizes what we care about the most -- the informative structure for network modeling. Therefore, our assumption can be phrased as an ``informative-core-noninformative-periphery" structure.

Under the proposed model, we develop spectral algorithms to identify the core structure with theoretically provable guarantees. In particular, we will show that our algorithms can exactly identify the core component even on sparse networks -- the so-called ``strong consistency" guarantee. The strong consistency is crucial in our context (compared with its ``weak consistency" cousin). This is because we design our method to be a general preprocessing step both in practice and theory. With strong consistency, the theoretical analysis for any downstream modeling of the core component remains valid by conditioning on the success of our method. On the contrary, such a seamless transition would not be available when only weak consistency is achieved. 

The rest of the paper is organized as follows. We first propose our core-periphery model in Section~\ref{sec:proposedModel} and then introduce the spectral methods for core identification under the proposed model in Section~\ref{secsec:algorithm}. Section~\ref{sec:theory} focuses on the theoretical properties of the algorithms with respect to the accuracy of core identification. Extensive evaluations are included in Section~\ref{sec:simulation}, where we demonstrate the advantage of our method against several benchmark methods for this problem. In Section~\ref{sec:data}, we demonstrate our method by extracting informative core structure from a citation network to improve downstream hierarchical community detection. We conclude the paper with discussions in Section~\ref{sec:conc}. All the proofs of theoretical results and additional simulation examples are included in the appendix.

%% file: model.tex
\section{Methodology}\label{sec:methodology}

{\bfseries Notations.} We use capital boldface letters such as $\bm{M}$ to denote matrices. Given a matrix $\bm{M}$, $\bm{M}_{i,*}$, $\bm{M}_{*,j}$, and $\bm{M}_{ij}$ are the $i$-th row, $j$-th column, and $(i,j)$-th entry, respectively. Let $\norm{ \bm{M} }_F$, $\norm{ \bm{M} }_2$, $\norm{ \bm{M} }_{2,\infty}$ be the Frobenius norm, the spectral norm, the two-to-infinity norm (maximum Euclidean norm of rows) of $\bm{M}$, respectively. In particular, we use $\bm{I}_d$ to denote the $d \times d$ identity matrix, and $\bm{1}_d$ to denote the $d \times 1$ vector whose entries are all $1$. Let $\rank(\bm{M})$ be the rank of $\bm{M}$, and $\bm{M}^t$ be the transpose of $\bm{M}$.  Let $[ l ]$ be the index set $\{1,2,...,l\}$. Let $\mathbb{O}_{p_1,p_2}$ be the set of $p_1 \times p_2$ matrices with orthonormal columns, and let $\mathbb{O}_{p}$ be the shorthand for $\mathbb{O}_{p,p}$. For any two positive sequences $\{a_n\}$ and $\{b_n\}$, we say $a_n \preceq b_n$ if there exists a positive constant $C$ such that $a_n \le C b_n$ for sufficiently large $n$; $a_n \succeq b_n$ if $-a_n \preceq b_n$; $a_n \simeq b_n$ if $a_n \succeq b_n$ and $a_n \preceq b_n$; $a_n \succ b_n$ if for an arbitrarily large $C>0$, $a_n > C b_n$ for sufficiently large $n$.

\subsection{A core-periphery model based on informative component}\label{sec:proposedModel}

Assume the network size to be $n$. We will focus on undirected and unweighted networks without self-loops. Such a network can be represented by an $n\times n$ symmetric binary adjacency matrix $\bm{A}$ such that $\bm{A}_{ij}$ is 1 if and only if node $i$ and $j$ are connected. 
We will embed our discussion in the following probabilistic framework for $\bm{A}$, which can be seen as a conditional version of the Aldous-Hoover representation when the network nodes are exchangeable \citep{aldous1981representations,hoover1979relations}. Specifically, we assume that there exists an underlying $n\times n$ probability matrix $\bm{P}$ such that $\bm{A}_{ij} \sim \text{Bernoulli}(\bm{P}_{ij})$, for $1 \le i < j \le n$ independently. We denote by $\bm{E}$ the difference between $\bm{A}$ and $\bm{P}$, i.e. $\bm{A} = \bm{P} + \bm{E}$. The elements $\{\bm{P}_{ij}\}$ are called edge probabilities or connection probabilities. The matrix $\bm{P}$ fully specifies the structural information of the network.

In our context, the periphery component should not admit structures that may be interesting for modeling. Though whether a particular type of structure is interesting may depend on specific applications, we believe the widely regarded \emph{uninteresting} pattern is relatively easy to define. The following core-periphery structure is defined according to one such pattern for the periphery. 

\begin{definition}[The ER-type core-periphery structure]\label{def:core_periphery1}
The nodes in the network can be partitioned into a core set $\ccal$ and a periphery set $\pcal$, where 
$$\pcal = \{i \in [n]| \bm{P}_{ij} = \bm{P}_{ik}, \text{ for all } j,k \in [n], j \ne i, k \ne i\}.$$
and $\ccal = [n]/ \pcal$.
\end{definition}
Note that due to symmetry of $\bm{P}$, \Cref{def:core_periphery1} indicates that all edges involving periphery nodes are generated randomly with the same probability resembling the  Erd\"{o}s-R\'{e}nyi (ER) model \citep{erdos1959random}. The subnetwork of the core,  in contrast, can follow any connection pattern as long as it is different from the periphery. Such generality on the core structure renders the flexibility to use our model as a data preprocessing step for any downstream analysis. In the special case when the core subnetwork is also an ER model but with a different density from the periphery part, the model reduces to the block model core-periphery structure used in \citet{borgatti2000models}, \citet{zhang2015identification}, and \citet{priebe2019two}. \Cref{fig:CP_const} shows one example of the core-periphery structure following \Cref{def:core_periphery1}.

The ER-type periphery is arguably the most basic form of non-informative structure. It also indicates that the periphery nodes should have similar degrees. In many settings, even if the nodes have heterogeneous degrees, their connection patterns may not be interesting either. One way to define such variation of the uninteresting connection only depends on two nodes separably, as defined next.

\begin{definition}[The configuration-type core-periphery structure]\label{def:core_periphery2}
Let $d_i$ be the expected degree of node $i$. 
The nodes in the network can be partitioned into a core set $\ccal$ and a periphery set $\pcal$, where 
\begin{equation}\label{eq:configuration-probability}
\pcal = \{i \in [n]| \bm{P}_{ij} = \frac{d_i d_j}{\sum_{k=1}^{n}d_k}, \text{ for all } j \in {\cal V}, j \ne i\}.
\end{equation}
and $\ccal = [n]/ \pcal$.
\end{definition}

The periphery connection pattern under \Cref{def:core_periphery2} essentially assumes $\bm{P}_{ij} \propto d_id_j$ for any pair involving at least one periphery node. Such a pattern resembles the configuration model \citep{bollobas1980probabilistic,chung2002average,newman2018configuration}, where the connection probability between two nodes is based on the degree of the two nodes. \Cref{fig:CP_config} illustrates this definition. Compared with the ER-type periphery, the periphery also exhibits a heterogeneous connection pattern. This model can adopt arbitrary degree distributions for the periphery nodes. 

\begin{figure}
\begin{centering}
\begin{subfigure}{.29\textwidth}
  \centering
  \includegraphics[width=1\textwidth]{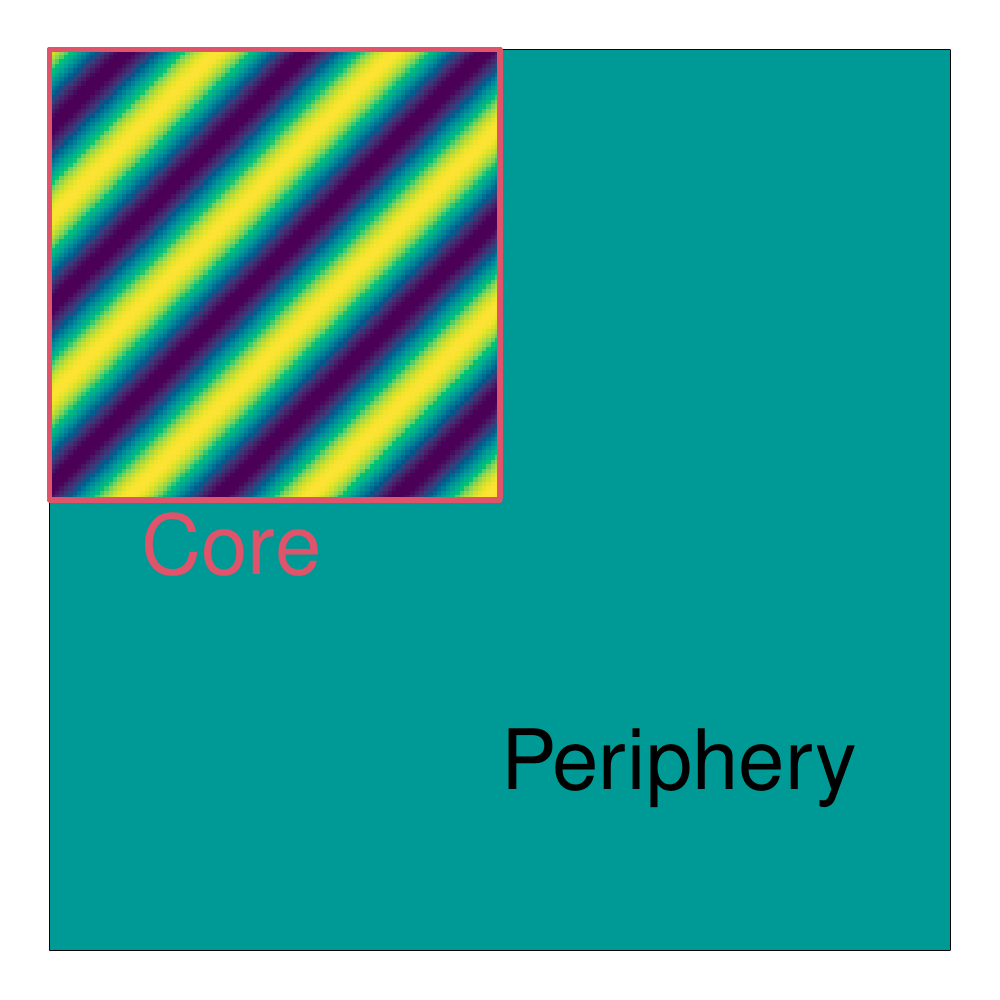}
  \caption{ER-type}
  \label{fig:CP_const}
\end{subfigure}%
\begin{subfigure}{.29\textwidth}
  \centering
  \includegraphics[width=1\textwidth]{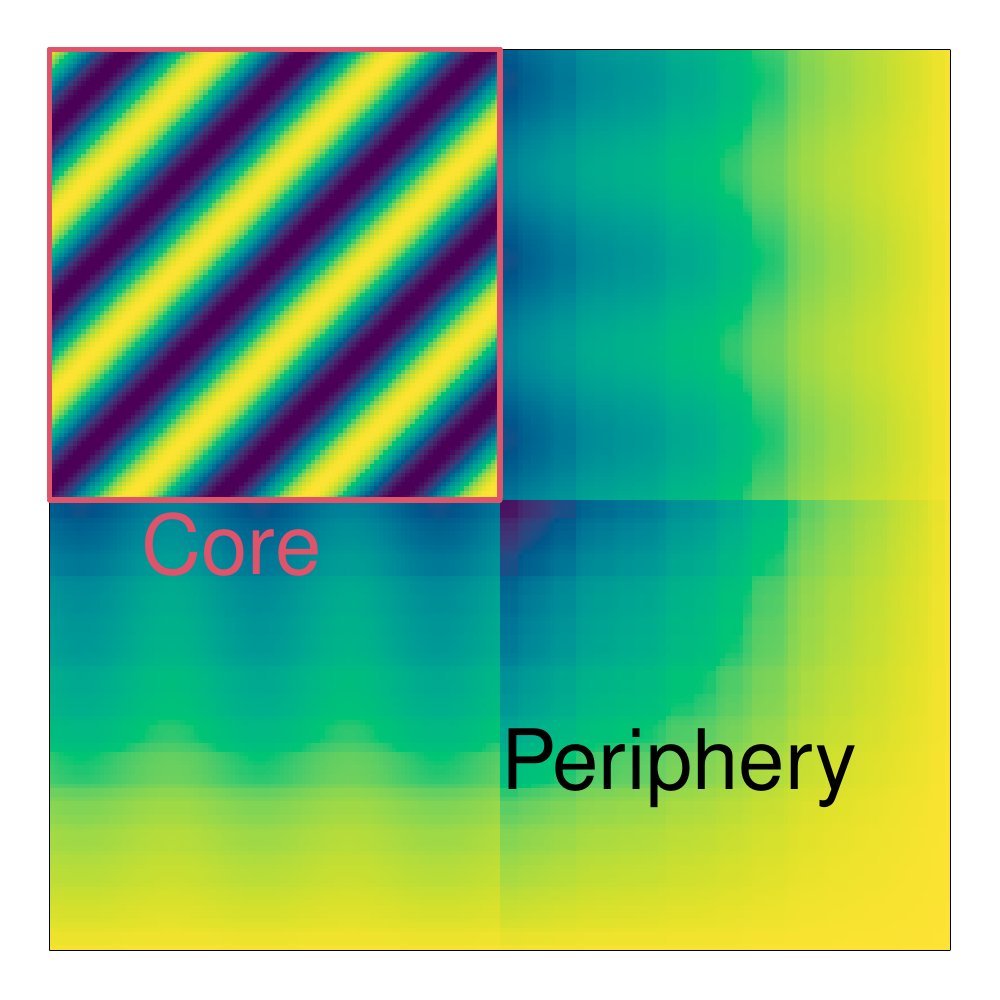}
  \caption{Configuration-type}
  \label{fig:CP_config}
\end{subfigure}%
\begin{subfigure}{.41\textwidth}
  \centering
  \includegraphics[width=1\textwidth]{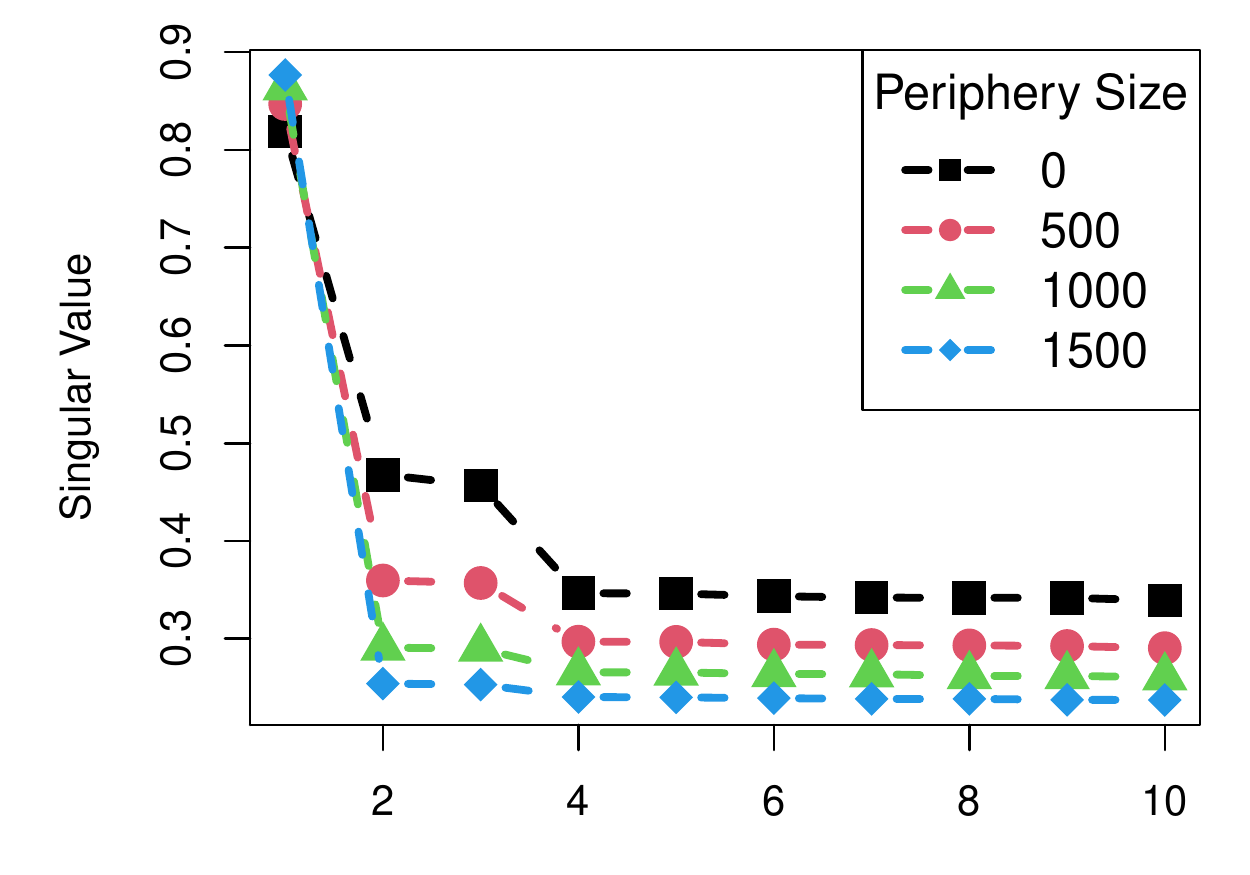}
  \caption{Impacts on eigengap}
  \label{fig:CP_Example_11}
\end{subfigure}
 \caption{Illustrations of our core-periphery models and the impacts of including the periphery component in the analysis. (a) The ER-type core-periphery model, where the expected degrees of the periphery nodes are constant. (b) The configuration-type core-periphery structure, where the expected degrees of the periphery nodes are randomly sampled from a uniform distribution. (c) The impact on the eigengap of the model by including periphery nodes following the ER-type model. The core model has rank 3, but including too many periphery nodes would overwhelm the signal, so all eigenvalues except the largest one become negligible.  }
  \end{centering}
\end{figure}



\subsection{Spectral algorithms for core identification}\label{secsec:algorithm}

We proceed to introduce our algorithms to identify the core (and periphery) components under the models of \Cref{def:core_periphery1} and \Cref{def:core_periphery2}. The likelihood-based procedures will not be applicable in the current context because we do not assume any specific model for the core subnetwork. Instead, we will resort to spectral methods for our purpose. Spectral methods have been used extensively in fitting various network models \citep{rohe2011spectral,sussman2012consistent,jin2015fast,lei2015consistency,qin2013regularized,ma2020universal,lei2020consistency,li2020community,wang2020spectral}, which also has the advantage of computational efficiency easy implementation. The crucial step to design such an algorithm is to find the desired spectral properties to leverage. Next, we will describe our algorithms for the ER-type model and configuration-type model separately. 

Under the ER-type model (\Cref{def:core_periphery1}), for any periphery node $i$, $\bm{P}_{i,*}$ is a vector of the same value except for the diagonal entry; for any core node $i$, the entries in $\bm{P}_{i,*}$ exhibit a large variation. Therefore, the core and periphery may be split according to the variation of entries in $\bm{P}_{i,*}$.  Define the centering matrix $\bm{H}$ to be $\bm{I}_n-\frac{1}{n}\bm{1}_n\bm{1}_n^t$. Then ${|| \bm{P}_{i,*}\bm{H} ||}_2^2$ is the squared total variation of the entries in $\bm{P}_{i,*}$. In particular, the norm ${|| \bm{P}_{i,*}\bm{H} ||}_2$ is almost zero for $i \in {\cal P}$, since $\bm{P}_{i,*}$ is a constant vector except on the $i$th coordinate. The periphery nodes can thus be identified for small ${|| \bm{P}_{i,*}\bm{H} ||}_2$ values.

In practice, when we only observe $\bm{A}$ instead of $\bm{P}$, the above strategy would not work due to the large perturbation of $\bm{A}$ from $\bm{P}$. The solution to this difficulty is denoising $\bm{A}$ by an estimator $\hat{\bm{P}}$ and applying the above procedure to $\hat{\bm{P}}$. Notice that $\rank(\bm{P}) \le \rank(\bm{P}^{\ccal}) + 1$ where $\bm{P}^{\ccal}$ is the model for the core subnetwork. Similar properties can be obtained for many reasonable definitions of stable rank. On the other hand, as studied in \citet{chatterjee2015matrix}, almost all interesting network models give approximately low-rank structure. These motivate us to consider $\bm{P}$ as \emph{approximately low-rank} (to be formally defined in our theory) and use some singular value truncating/thresholding estimator as $\hat{\bm{P}}$.  The simplest estimator would be the universal singular value thresholding method of \citet{chatterjee2015matrix}. However, theoretically and empirically, using an adaptive way to cut off the singular values of $\bm{A}$ to a certain rank turns out to be more effective. Specifically, given a positive integer $r$, we use the rank-$r$ truncated SVD of $\bm{A}$ as $\hat{\bm{P}}$. Our algorithm for Erd\"{o}s-Renyi periphery defined in \Cref{def:core_periphery1} is summarized in \Cref{alg:cp_detection1}. In the algorithm, we treat the approximating rank $r$ as given. In practice,  The $r$ will be selected according to data-driven methods. In this paper, we always use the cross-validation method of \citet{li2020network} to select a proper $r$, which can be see as a procedure the select the best low-rank approximation for link predictions.

\begin{algorithm}[tb]
   \caption{Spectral algorithm for core identification from the ER-type periphery} 
   \label{alg:cp_detection1}
\begin{algorithmic}
   \STATE {\bfseries Input:} The adjacency matrix $\bm{A}$, the core size $N_{\cal C}$ and approximating rank $r$.
   \begin{enumerate}
   \item Find the low-rank approximation of $\bm{A}$ through rank $r$ truncated SVD. Denote the resulting matrix by $\hat{\bm{P}}$.
   \item Compute the score $S_i = {|| \hat{\bm{P}}_{i,*} \bm{H} ||}_2$, for $i \in [n]$.
   \item Sort the scores $S_1, S_2, ..., S_n$.
   \item For each $i \in [n]$, classify node $i$ as a core node if $S_i$ is among the top-$N_{\cal C}$ scores; otherwise classify node $i$ as a periphery node.
   \end{enumerate}
\end{algorithmic}
\end{algorithm}

Under the configuration-type core-periphery model (\Cref{def:core_periphery2}), a similar strategy can be applied with an additional modification. The key ingredient is a degree-correction step to neutralize the impacts of heterogeneous degrees. According to the periphery connection probabilities in \eqref{eq:configuration-probability}, for any $i \in \pcal$, we have
$$\bm{P}_{ij}/d_j = \frac{d_i}{\sum_k d_k}, \text{~~ for any~~} j \ne i.$$
Hence, normalizing the columns by the corresponding degrees would result in a the matrix in which the row for each periphery node is a constant, except for the diagonal entry. Define $\bm{D} = \diag(d_1, \cdots, d_n)$. The column correction step can be written as $\bm{P}\bm{D}^{-1}$.  After this degree-correction step, the same idea in \Cref{alg:cp_detection1} can be applied here and we will use ${|| \bm{P}_{i,*}\bm{D}^{-1}\bm{H} ||}_2$ to separate the core nodes from the periphery nodes. In practice, $\bm{P}$ is again substituted by its estimate $\hat{\bm{P}}$, and $\bm{D}$ is replaced by its sample version $\hat{\bm{D}}$. The details are summarized in \Cref{alg:cp_detection2}.

\begin{algorithm}[tb]
   \caption{Spectral algorithm for core identification from the configuration-type periphery}
   \label{alg:cp_detection2}
\begin{algorithmic}
   \STATE {\bfseries Input:} The adjacency matrix $\bm{A}$, the core size $N_{\cal C}$ and approximating rank $r$.
   \begin{enumerate}
   \item Find the low-rank approximation of $\bm{A}$ through rank $r$ truncated SVD. Denote the resulting matrix by $\hat{\bm{P}}$.
   \item Compute $\hat{d}_i = \sum_{j=1}^{n}\bm{A}_{ij}$, and let $\hat{\bm{D}} = \text{diag}\{\hat{d}_1, \hat{d}_2, ..., \hat{d}_n\}$.
   \item Compute $S_i' = {|| \hat{\bm{P}}_{i,*} \hat{\bm{D}}^{-1} \bm{H} ||}_2$, for $i \in [n]$.
   \item Sort scores $S_1', S_2', ..., S_n'$.
   \item For each $i \in [n]$, classify node $i$ as a core node if $S_i'$ is among the top-$N_{\cal C}$ scores; otherwise classify node $i$ as a periphery node.
   \end{enumerate}
\end{algorithmic}
\end{algorithm}

As can be seen, the major computational burden of Algorithm~\ref{alg:cp_detection1} and \ref{alg:cp_detection2} is on the SVD of $\bm{A}$, which is highly efficient. Thus both of the algorithms are scalable to large networks. Moreover, in the next section, we will show that these algorithms can accurately identify the core nodes even on sparse networks.

%% file: theory.tex
\section{Theoretical properties}\label{sec:theory}

This section will introduce a few theoretical results about the accuracy of core identification by our spectral algorithms. We will start from the ER-type model, and then the same set of theoretical properties will be extended to the configuration-type model. 

\subsection{Theory under the ER-type model}\label{secsec:ER-theory}

The success of \Cref{alg:cp_detection1} depends on the magnitude of $\norm{ \bm{P}_{i,*}\bm{H} }_2$ for core nodes. To quantify this magnitude, additional notations have to be introduced. First, define 
$$h(n) = \min_{i \in {\cal C}} \norm{ \bm{P}_{i,*} \bm{H} }_2, \text{~~and~~} p^* = \max_{1 \le i,j \le n} \bm{P}_{ij}.$$

Our algorithms also relies on a good estimate of the probability matrix $\hat{\bm{P}}$. As mentioned in the previous section, we will use the rank-$r$ truncated SVD of the observed adjacency matrix $\bm{A}$ as  $\hat{\bm{P}}$. Suppose $\bm{P}$ and $\bm{A}$ admit the following eigen-decompositions:
\begin{equation}\label{eq:svd_P}
\bm{P} = \begin{bmatrix}
\bm{U} & \bm{U}_{\bot}
\end{bmatrix}
\begin{bmatrix}
\bm{\Lambda} & \bm{0} \\
\bm{0} & \bm{\Lambda}_{\bot}
\end{bmatrix}
\begin{bmatrix}
\bm{U}^t \\
\bm{U}_{\bot}^t
\end{bmatrix} = \bm{U} \bm{\Lambda} \bm{U}^t + \bm{U}_{\bot} \bm{\Lambda}_{\bot} \bm{U}_{\bot}^t,
\end{equation}
\begin{equation}\label{eq:svd_A}
\bm{A} = \begin{bmatrix}
\hat{\bm{U}} & \hat{\bm{U}}_{\bot}
\end{bmatrix}
\begin{bmatrix}
\hat{\bm{\Lambda}} & \bm{0} \\
\bm{0} & \hat{\bm{\Lambda}}_{\bot}
\end{bmatrix}
\begin{bmatrix}
\hat{\bm{U}}^t \\
\hat{\bm{U}}_{\bot}^t
\end{bmatrix}
= \hat{\bm{U}} \hat{\bm{\Lambda}} \hat{\bm{U}}^t + \hat{\bm{U}}_{\bot} \hat{\bm{\Lambda}}_{\bot} \hat{\bm{U}}_{\bot}^t,
\end{equation}
where $\bm{\Lambda} = \text{diag}\{ \lambda_1, \lambda_2, ..., \lambda_r \}$ and $\bm{\Lambda}_{\bot} = \text{diag}\{ \lambda_{r+1}, \lambda_{r+2}, ..., \lambda_n \}$ consist of the eigenvalues of $\bm{P}$ sorted in decreasing order. $\bm{U} \in \mathbb{O}_{n,r}$ and $\bm{U}_{\bot} \in  \mathbb{O}_{n,n-r}$ contain corresponding eigenvectors as columns. The matrices $\hat{\bm{\Lambda}}$, $\hat{\bm{\Lambda}}_{\bot}$, $\hat{\bm{U}}_{\bot}$ and $\hat{\bm{U}}_{\bot}$ are similarly defined for $\bm{A}$. Our estimator of $\bm{P}$ is $\hat{\bm{P}}=\hat{\bm{U}} \hat{\bm{\Lambda}} \hat{\bm{U}}^t$. For such a low-rank approximation to work well, we will impose the following assumptions:

\begin{assumption}[Approximate low-rankness]\label{assumption:low_rank}
$| \lambda_r | \succeq \frac{n p^*}{\sqrt{r}}$, and $| \lambda_{r+1} | \preceq \sqrt{p^* \log{n}}$.
\end{assumption}

\begin{assumption}[Incoherence]\label{assumption:incoherence}
$\norm{ \bm{U} }_{2,\infty} \le \mu_0 \sqrt{\frac{r}{n}}$, for a scalar $\mu_0$ that may depend on $n$.
\end{assumption}

\Cref{assumption:low_rank} above is about the gap between the $r$th and $(r+1)$th eigenvalues, which is needed for low-rank approximation to be reasonable. Notice that the condition implicitly requires that $p^* \succeq \frac{r \log{n}}{n^2}$, which eliminates extremely sparse network models such as bounded-degree networks. However, as can be seen later on, such a requirement is trivial and will be overwritten by a stronger density requirement for a valid network concentration. The incoherence condition ensures that the entries of $\bm{P}$ spread out evenly across all nodes. Such an assumption is widely used in matrix completion and random matrix literature \citep{candes2009exact,chen2015incoherence,fan2018eigenvector,cape2019two,abbe2020entrywise}, and is generally considered necessary for highly accurate entrywise or row/column-wise recovery of random matrices.


\begin{theorem}\label{cor:strong_consistency_1}
Assume the network $\bm{A}$ is generated from the ER-type model in \Cref{def:core_periphery1}, under \Cref{assumption:low_rank} and \Cref{assumption:incoherence}. \Cref{alg:cp_detection1} is used to identify the core nodes with the correct $N_{C}$ and $r$. Furthermore, suppose $p^* \succeq \max\left\{ \frac{ \mu_0^2 r \log{n} }{n}, \frac{\mu_0^2 r^2}{n} \right\} $, and $\left| \lambda_1 / \lambda_r \right|$ are bounded. If
\begin{equation}\label{eq:cor_strong_consistency_1}
h(n) \succeq \mu_0 \sqrt{ r (\log{n} + r) p^* } + \mu_0^2 r \sqrt{ p^* } ,
\end{equation}
then, for sufficiently large $n$, \Cref{alg:cp_detection1} exactly identifies the core and periphery nodes with probability at least $1-(B(r)+2)n^{-\gamma}$ for some positive constant $\gamma$, where $B(r) = 10 \min\{r, 1+\log_2(|\lambda_1/\lambda_r|)\}$.
\end{theorem}

We present \Cref{cor:strong_consistency_1} under the approximately low-rank condition of \Cref{assumption:low_rank} for conciseness. The assumption can be further relaxed. The more general version of the theorem is included in \Cref{appendix:proof_strong_consistency_1}. Notice that we do not assume that the density of the core subnetwork is denser than the periphery. Nor do we have to assume that the core size is in the same order as the periphery size, though the sizes' impacts are implicitly considered in $h(n)$. Such generality gives our method significant advantages in practice, as demonstrated later in Section~\ref{sec:simulation} and \ref{sec:data}.

To illustrate condition \eqref{eq:cor_strong_consistency_1}, consider the stochastic block model (SBM) as an example of the core structure under the ER-type model. Specifically, we write the model as a $K$-block model with the first $K-1$ clusters as the core and the last cluster as the periphery where $K$ if a fixed integer. Assume all the $K$ cluster sizes are in the same order. Let $p^* \bm{B}$ be the symmetric $K \times K$ matrix of connection probabilities between the $K$ blocks,  where $\bm{B}$ is a fixed matrix with maximum entry $1$, and all entries in the last row and column being equal. Let $\bm{Z}$ be a $n \times K$ membership matrix where $\bm{Z}_{ik}=1$ if and only if node $i$ belongs to block $k$. Ignoring the no-self-loop constraint for simplicity, we have $\bm{P} = p^*\bm{Z}\bm{B}\bm{Z}^t$. In this case, $h(n)$ is at the order of $p^* \sqrt{n}$, and the conditions in \eqref{eq:cor_strong_consistency_1} becomes $p^* \succeq \frac{\log{n}}{n}$, corresponding to the requirement of average expected node degree $\succeq \log n$. In this setting, if we treat splitting the first $K-1$ blocks and the last block as a bicluster problem, our requirement corroborates the network sparsity requirement of \citet{abbe2020entrywise} for the spectral clustering algorithm under the two-block SBM.

In practice, the number of core nodes, $N_{\cal C}$, is often unknown. However, under a slightly stronger condition than \Cref{cor:strong_consistency_1}, we can calculate a threshold such that the correct $N_{\ccal}$ can be recovered by cutting off the scores in \Cref{alg:cp_detection1}. In particular, define $\hat{p} = \frac{2}{n^2-n}\sum_{i<j} \bm{A}_{ij}$ and replace the $N_{\ccal}$ in Step 4 of \Cref{alg:cp_detection1} by  
\begin{equation}\label{eq:Nhat-1}
\hat{N}_{\ccal} = |\{i: S_i > \sqrt{\hat{p}^{1-\epsilon}\log{n}}\}|
\end{equation}
for some small constant $\epsilon$. In all of our experiments, we use $\epsilon=0.01$. The same type of performance as \eqref{eq:cor_strong_consistency_1} can still be theoretically guaranteed by this thresholding strategy.
\begin{corollary}\label{prop:threshold_1}
Under the conditions of \Cref{cor:strong_consistency_1}, suppose $\mu_0$ and $r$ are bounded. Furthermore, assume
$$\min_{1 \le i,j \le n} \bm{P}_{ij} \simeq \max_{1 \le i,j \le n} \bm{P}_{ij} = p^*, $$
and 
$$h(n) \succ \sqrt{ {p^*}^{(1-\epsilon)} \log{n}}$$
 for the constant $\epsilon$ in \eqref{eq:Nhat-1}. If the $\hat{N}_{\ccal}$ defined by \eqref{eq:Nhat-1} is used in  \Cref{alg:cp_detection1}, with sufficiently large $n$, the core and periphery can be exactly identified with probability at least $1-(B(r)+4)n^{-\gamma}$ for some positive constant $\gamma$.
\end{corollary}

We conclude this section by providing an upper bound for the number of misidentified core nodes under weaker assumptions. 
\begin{theorem}\label{thm:weak_consistency_1}
Assume the network $\bm{A}$ is generated from the ER-type model in \Cref{def:core_periphery1}, and \Cref{alg:cp_detection1} is used to identify the core nodes with the correct $N_{\cal C}$. Suppose $h(n) > p^*$. Denote the number of misclassified core nodes by $M$. For a sufficiently large $n$, we have
\begin{equation}\label{eq:weak-consistency-ER}
M \preceq \max\{r, \rank(\bm{P})\} \cdot \frac{ {\left( \max\{\sqrt{n p^*},\sqrt{\log{n}}\} + | \lambda_{r+1} | \right)}^2 }{ {\left(h(n)-p^*\right)}^2 }
\end{equation}
with probability at least $1-n^{-\gamma}$ for some positive constant $\gamma$.
\end{theorem}

For illustration, consider the SBM example after Theorem~\ref{cor:strong_consistency_1} again with $p^* \ge \log{n}$. In this case, \eqref{eq:weak-consistency-ER} indicates that the misidentified number is upper bounded by $K/p^* \le Kn/\log{n} = o(n)$. Such a vanishing proportion of misidentified core nodes is also called the ``weak consistency". However, compared with the strong consistency of Theorem~\ref{cor:strong_consistency_1} , the weak consistency is less useful in our scenario. This is because, as a general data preprocessing step, having strong consistency in our method ensures that the downstream theoretical analysis can still go through as if the core is already given. The weak consistency, in contrast, loses this possibility, and the downstream analysis has to consider the potential errors of the core identification and the potential dependence introduced by this preprocessing step.

\subsection{Theory under the configuration-type model}\label{secsec:configuration-theory}

Next, we consider the configuration-type model following \Cref{def:core_periphery2}. Recall that for a periphery node $i$, $\bm{P}_{i,*} \bm{D}^{-1}$ is a constant vector except for the diagonal entry. Therefore, the proof can be done by applying the same strategy of last section on the degree corrected version of $\bm{P}$. Define
$$ h'(n) = \min_{i \in {\cal C}} {|| \bm{P}_{i,*} \bm{D}^{-1} \bm{H} ||}_2 . $$
Under the configuration-type model, the quantity $h'(n)$ has a similar role to the $h(n)$ for the ER-type model. 
\begin{theorem}\label{cor:strong_consistency_2}
Assume the network $\bm{A}$ is generated from the configuration-type model in \Cref{def:core_periphery2}, under \Cref{assumption:low_rank} and \Cref{assumption:incoherence}. \Cref{alg:cp_detection2} is used to identify the core nodes with the correct $N_{\cal C}$ and $r$. Let $d_{\min} = \min_{1 \le i \le n} \sum_{j=1}^{n}\bm{P}_{ij}$, and suppose $d_{\min} \succ \log{n}$, $p^* \succ \max\left\{ \frac{\mu_0^2 r \log{n} }{n}, \frac{\mu_0^2 r^2}{n} \right\} $, and $\left| \lambda_1 / \lambda_r \right|$ is bounded. If
\begin{equation}\label{eq:strong_consistency_2}
h'(n) \succ \frac{ 1 }{ d_{\min} }\left( \mu_0 \sqrt{ r (\log{n} + r) p^* } + \mu_0^2 r \sqrt{ p^* } \right) + \norm{\bm{P}\bm{D}^{-1}}_{2,\infty} \sqrt{\frac{ \log{n} }{d_{\min}}} ,
\end{equation}
then, for sufficiently large $n$, \Cref{alg:cp_detection2} exactly identifies the core and periphery nodes with probability at least $1-(B(r)+4)n^{-\gamma}$, where $B(r) = 10 \min\{r, 1+\log_2(|\lambda_1/\lambda_r|)\}$.
\end{theorem}
Again, a more general version of the theorem is provided in the \Cref{appendix:proof_strong_consistency_2}. To illustrate the condition \eqref{eq:strong_consistency_2}, we consider the example when the degree-corrected stochastic block model (DC-SBM) \citep{karrer2011stochastic} is true core model. Specifically, assume that the whole network follows the DC-SBM with the first $K-1$ clusters being the core while the last cluster being the periphery. Suppose all clusters have equal size, and $K$ is fixed. Let $z_i \in \{1, \cdots, K\}$ be the cluster label of node $i$. The model can be parametrized by a sequence of node popularity parameters $\theta_i, 1\le i\le n$ and a $K\times K$ matrix $\rho \bm{B}$ where $\bm{B}$ is a fixed symmetric matrix with the last row and column containing only 1's and $\rho$ depends on $n$. The connection probability of this DC-SBM is given by $P_{ij}=\theta_i\theta_j\rho \bm{B}_{z_iz_j}$. To ensure the identifiability of the model, we use the constraint of \citet{zhao2012consistency}: $\sum_{z_i=k}\theta_i = n/K.$ Furthermore, assume that $\bm{B}$ satisfies $\sum_{k'}\bm{B}_{kk'} = K, 1\le k\le K-1$, it can be verified that this model satisfies  \Cref{def:core_periphery2}. Under this model, in the simplified setting such that $\mu_0$ is bounded, $r=K$, and $\theta_i \simeq 1$ for all $i$, the condition \eqref{eq:strong_consistency_2} reduces to the degree requirement of $d_{\min} \succ \log{n}$.

Similar to the case of the ER-type model, when $N_{\cal C}$ is unknown, a threshold to cut off scores can be used to determine the core-periphery separation under slightly stronger conditions. Recall that $\hat{p} = \frac{2}{n^2-n}\sum_{i<j} \bm{A}_{ij}$. We can replace the $N_{\ccal}$ in Step 5 of \Cref{alg:cp_detection2} by  
\begin{equation}\label{eq:Nhat-2}
\hat{N}_{\ccal}' = | \{ i: S_i' > \frac{ \sqrt{ \log{n} }}{ n \sqrt{ \hat{p}^{1+\epsilon} } } \} |
\end{equation}
for some small constant $\epsilon$. In all of our experiments, we use $\epsilon=0.01$.
\begin{corollary}\label{prop:threshold_2}
Under the conditions of \Cref{cor:strong_consistency_2}, suppose $\mu_0$ and $r$ are bounded. Furthermore, assume
$$\min_{1 \le i,j \le n} \bm{P}_{ij} \simeq \max_{1 \le i,j \le n} \bm{P}_{ij} = p^*, $$
and 
$$ h'(n) \succ \frac{ \sqrt{ \log{n} }}{ n \sqrt{ {p^*}^{1+\epsilon} } } $$
 for the constant $\epsilon$ in \eqref{eq:Nhat-2}. If the $\hat{N}_{\ccal}'$ defined by \eqref{eq:Nhat-2} is used in  \Cref{alg:cp_detection2}, with a sufficiently large $n$, the core and periphery nodes can be exactly identified with probability at least $1-(B(r)+6)n^{-\gamma}$ for some positive constant $\gamma$.
\end{corollary}

Finally ,the following result is still available under weaker conditions.
\begin{theorem}\label{thm:weak_consistency_2}
Assume the network $\bm{A}$ is generated from the configuration-type model in \Cref{def:core_periphery2}, and \Cref{alg:cp_detection2} is used to identify the core nodes with the correct $N_{\cal C}$. Suppose $d_{\min} \succ \log{n}$, and $h'(n) > \frac{d_{\max}}{(n-1)d_{\min}}$. Denote the number of misclassified core nodes by $M'$. Then,
$$M' \preceq \max\{r, \rank(\bm{P})\} \cdot \frac{ np^* + \lambda_{r+1}^2 + \norm{\bm{P}\bm{D}^{-1}}_2^2 \cdot d_{\min} \cdot \log{n} }{ d_{\min}^2 \left[h'(n) - \frac{d_{\max}}{(n-1)d_{\min}} \right]^2 }$$
with probability at least $1-\frac{3}{n^{\gamma}}$ for some positive constant $\gamma$.
\end{theorem}

%% file: simulation.tex
\section{Simulation examples}\label{sec:simulation}

In this section, we evaluate the performance of our proposed algorithm on finite-size synthetic networks. We will demonstrate the effectiveness and the advantage of our method under a few different core models and density gaps between the core the periphery.

In generating our networks, we always set the first $N_{\ccal}$ nodes to core. To demonstrate the flexibility with respect to the core structure, we set the core component according to the graphon models \citep{aldous1981representations}. Specifically, the core submatrix $\bm{P}^{\ccal}$ is generated in the following way. Given a graphon function $g:[0,1] \times [0,1] \to [0,1]$, we first generate $N_{\ccal}$ i.i.d. random variables $\xi_i \sim \text{Uniform}[0,1], i = 1, \cdots, N_{\ccal}$, and then $\bm{P}^{\ccal}$ is set as
\begin{equation}\label{eq:graphon}
P^{\ccal}_{ij} = g(\xi_i, \xi_j), 1\le i,j \le N_{\ccal}
\end{equation}
We use three graphon functions defined in \citet{zhang2017estimating} as our simulation examples.   The first one gives the simplest SBM for $\bm{P}^{\ccal}$ with blockwise constant structure; The second one still has a low-rank $\bm{P}^{\ccal}$, but does not have the nice block structure; The third model is even more complicated and generates a full-rank $\bm{P}^{\ccal}$ -- this is a setting to verify the validity of our low-rank approximation strategy when the model is full-rank. The three models are summarized in \Cref{tab:graphons} and the heatmaps of the $\bm{P}^{\ccal}$ in the three models are shown in \Cref{fig:sim_graphon_Const,fig:sim_graphon_Config}.  Given $\bm{P}^{\ccal}$, we fill in the other positions of $\bm{P}$ by periphery probabilities. For the ER-type model, we simply fill in a constant value. For the configuration-type model, the construction involves multiple steps. Let $\theta^{\cal C}_i = \sum_{j=1}^{N_{\cal C}} \bm{P}^{\cal C}_{ij}$, and  sample $\theta^{\cal P}_i, i=1,2,...,N_{\cal P}$ from a uniform distribution between $0.5\min_{i\in \ccal}\theta_i$ and $1.5\max_{i\in \ccal}\theta_i$. Then, let $\bm{\theta} = \{\theta^{\cal C}_1, \theta^{\cal C}_2, ..., \theta^{\cal C}_{N_{\cal C}}, \theta^{\cal P}_1, \theta^{\cal P}_2, ..., \theta^{\cal P}_{N_{\cal P}} \} $. The edge probability involving periphery node is set as $\bm{P}_{ij} = \frac{\bm{\theta}_i \bm{\theta}_j}{ \sum_{k=1}^{N_{\cal C}} \theta^{\cal C}_k }$. It is not difficult to see that from this procedure, $d_i = \sum_{j=1}^{n} \bm{P}_{ij} = \frac{ \bm{\theta}_i \sum_{j=1}^{n} \bm{\theta}_{j} }{\sum_{k=1}^{N_{\cal C}} \theta^{\cal C}_k}$, and $\bm{P}_{ij} = \frac{ d_i d_j }{ \sum_{k_1}^{n} d_k }$ for $i \in {\cal P}$, matching \Cref{def:core_periphery2}.

We then rescale the generated probability matrix, so the average edge density is around $0.02$. In different configurations, we vary the average degrees of core and periphery nodes to demonstrate the effects of varying density ratio between the two components. We focus on the settings where the core has an equal or higher density than the periphery \footnote{Our methods perform well even if the core is sparser than the periphery. However, such a setting may be less realistic, so it is not included.}. The core size and periphery size are both $1000$ in this section. In \Cref{appendix:additional_simulation}, we also include results for setting of imbalanced sizes.

\begin{table}
\caption{Graphons for simulating network cores.}
\label{tab:graphons}
\centering
\begin{tabular}{c | c}
Graphon function $g(\mu, \nu)$ & Rank \\
\hline
$k/7$, if $\mu,\nu \in ((k-1)/6, k/6)$; $0.3/7$ otherwise. & $6$ \\ 
$\sin[5\pi(\mu+\nu-1)+1]/2+0.5$ & $3$ \\
$1/\{ 1+\exp{[15(0.8|\mu-\nu|)^{4/5}-0.1]} \}$ & Full \\
\end{tabular}
\end{table}

Several benchmark core-periphery identification methods are included in the evaluation. The first two methods are degree thresholding (Degree) and PageRank \citep{page1999pagerank} thresholding (PageRank). These two centrality measures are shown to be competitive for identifying the core component in the study of \citep{barucca2016centrality,rombach2017core}. Theoretically, it is shown by \citet{zhang2015identification} that under the SBM core-periphery model, the degree thresholding is optimal in favorable configurations. Another commonly used method is thresholding by the local clustering coefficient \citep{watts1998collective} (Local CC). The $k$-core pruning (k-core) algorithm \citep{seidman1983network} is also included in our evaluation. It can be seen as a more adaptive version than the degree thresholding and is shown to effectively extract meaningful subnetworks in \citet{wang2016discussion,li2020network,li2020hierarchical}. The final method is from \citet{priebe2019two}, where the Adjacency Spectral Embedding (ASE) \citet{sussman2012consistent} is used to capture the core-periphery structure when both affinity and core-periphery structures are present.

To fully characterize the core identification performance, we consider the tradeoff between the true positive rate (TPR) and the false positive rate (FPR), define as
$$\text{TPR} = \frac{\#\{\text{Correctly identified nodes}\}}{\#\{\text{Identified nodes}\}}~~\text{and}~~\text{TPR} = \frac{\#\{\text{Incorrectly identified  nodes}\}}{\#\{\text{Identified nodes}\}}.$$
These two metrics can be shown by the receiver operating characteristic (ROC). For each thresholding-based method, the full ROC curve is obtained if by varying the threshold. The k-core pruning is applied with $k$ increasing from 0 to the large integer, producing a sequence of points in the ROC space. The ASE, however, only gives a single point in the ROC space. For our method, we also include the single points based on our recommended threshold selection methods in Corollary~\ref{prop:threshold_1} and \ref{prop:threshold_2}, denoted by ``$\ast$". Empirically, we also found that applying $k$-means algorithm with $k=2$ to the log-transformed scores works well in our simulation, and we mark the point obtained this way by ``+" on the ROC curves. 

\Cref{fig:sim_graphon_Const} shows the results under the ER-type model. As can be seen, the easiest setting is when the core is much denser than the periphery. In this setting, most of the methods are reasonably good, and though our method is the most effective one, the advantage not moderate. As the density between the core and the periphery becomes more similar, the problem becomes more difficult, and some of the benchmarks become close to random guess. However, our method still maintains good performance, and the advantage over other methods becomes more significant. This is expected since many of the benchmarks rely on the density gap between the two components while our method does not. By comparing the results across different core models, one can see that the benchmark methods may perform well under one model but fails under another. In contrast, our method remains the best one in all settings, thanks to our model's generality. Finally, the thresholds given by our theory ($\ast$) and $k$-means clustering (+) render good model selections in the ROC space.

\begin{figure}[H]
\centering
\begin{subfigure}{0.96\textwidth}
  \centering
  \includegraphics[width=1\textwidth]{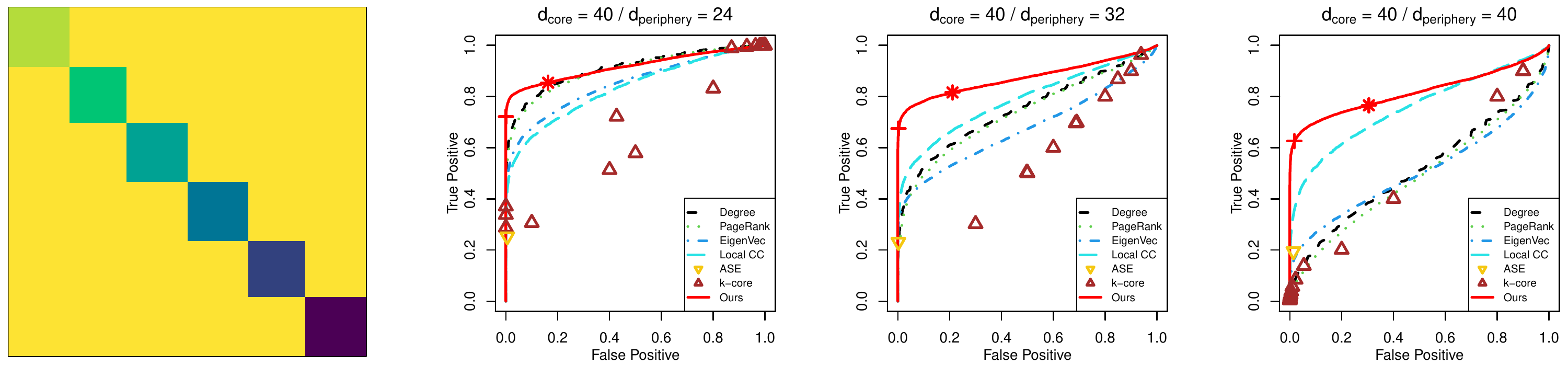}
  \caption{Graphon 1}
  \label{fig:graphon_1}
\end{subfigure} \\
\begin{subfigure}{0.96\textwidth}
  \centering
  \includegraphics[width=1\textwidth]{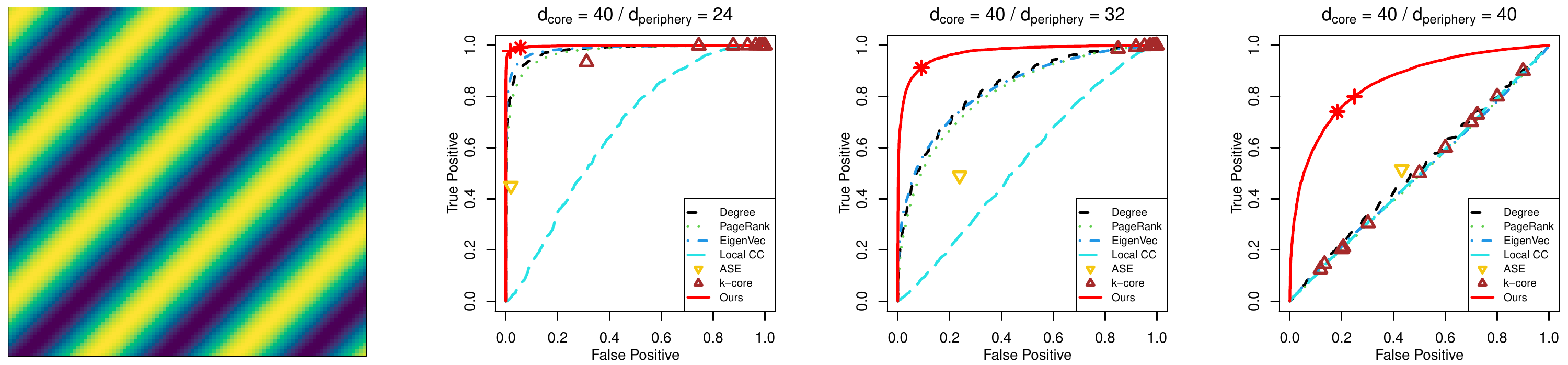}
  \caption{Graphon 2}
  \label{fig:graphon_2}
\end{subfigure} \\
\begin{subfigure}{0.96\textwidth}
  \centering
  \includegraphics[width=1\textwidth]{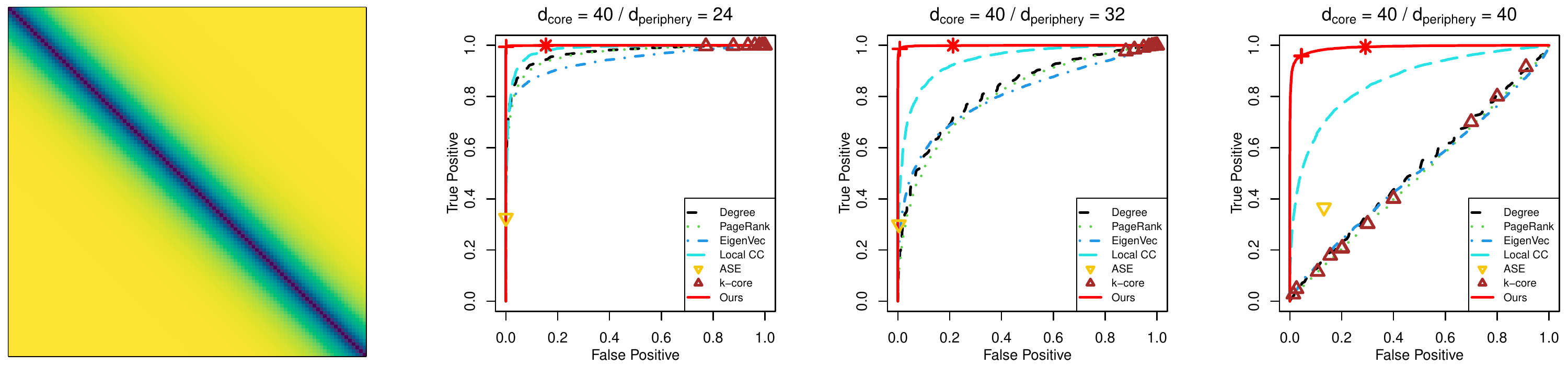}
  \caption{Graphon 3}
  \label{fig:graphon_3}
\end{subfigure} \\
\caption{Simulation results under ER-type core-periphery model where $N_{\ccal} = N_{\pcal} = 1000$. The left figures are the core graphon functions, and the corresponding ROC curves are shown on the right, under different degree-gaps between core and periphery. The point ``$\ast$" gives the model selection based on Corollary~\ref{prop:threshold_1}, and ``+"  indicates the model selection by $k$-means clustering with $k=2$.}
\label{fig:sim_graphon_Const}
\end{figure}

\begin{figure}[H]
\centering
\begin{subfigure}{0.96\textwidth}
  \centering
  \includegraphics[width=1\textwidth]{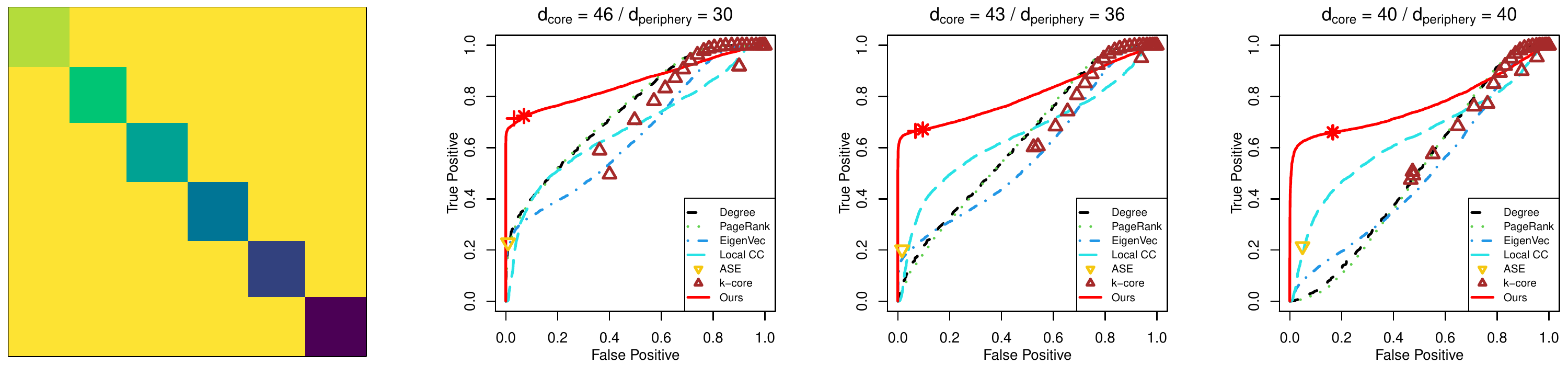}
  \caption{Graphon 1}
  \label{fig:graphon_1}
\end{subfigure} \\
\begin{subfigure}{0.96\textwidth}
  \centering
  \includegraphics[width=1\textwidth]{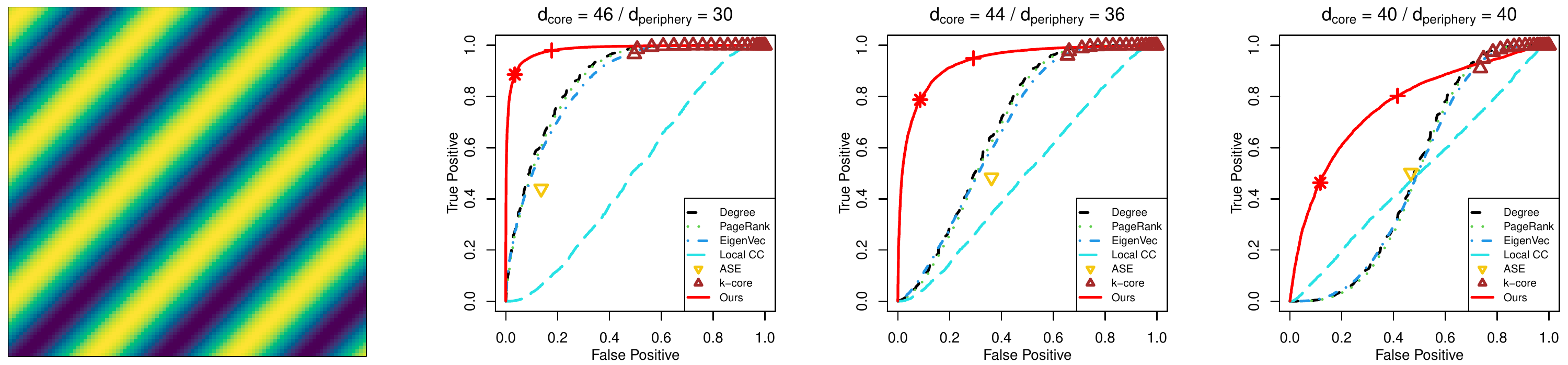}
  \caption{Graphon 2}
  \label{fig:graphon_2}
\end{subfigure} \\
\begin{subfigure}{0.96\textwidth}
  \centering
  \includegraphics[width=1\textwidth]{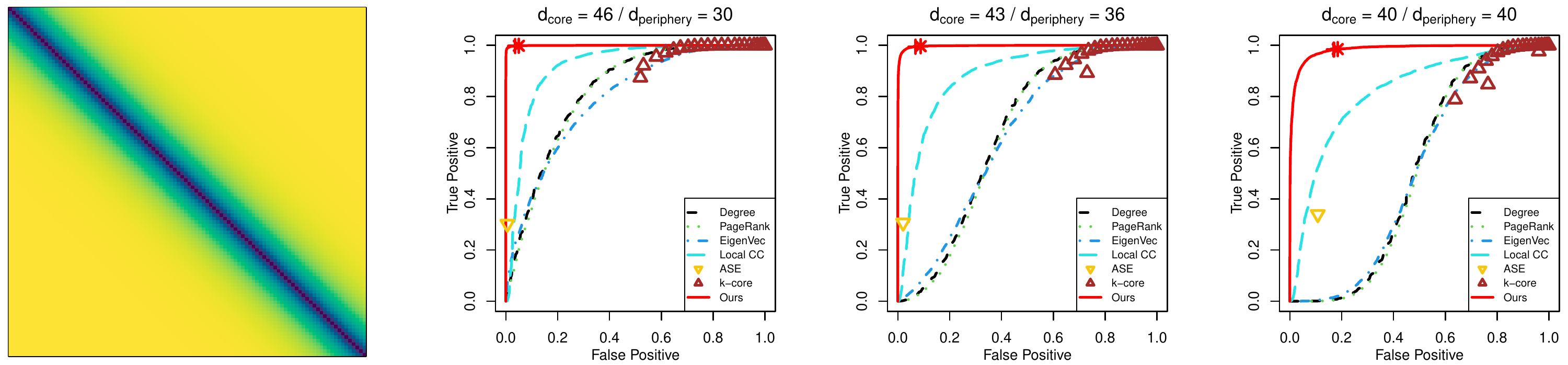}
  \caption{Graphon 3}
  \label{fig:graphon_3}
\end{subfigure} \\
\caption{Simulation results under configuration-type core-periphery model where $N_{\ccal} = N_{\pcal} = 1000$. The left figures are the core graphon functions, and the corresponding ROC curves are shown on the right, under different degree-gaps between core and periphery. The point ``$\ast$" gives the model selection based on Corollary~\ref{prop:threshold_2}, and ``+"  indicates the model selection by $k$-means clustering with $k=2$.}
\label{fig:sim_graphon_Config}
\end{figure}

\Cref{fig:sim_graphon_Config} shows the results under the configuration-type model. The pattern is very similar to that of \Cref{fig:sim_graphon_Const}.  Overall, the simulation examples show that our methods outperform the benchmark methods in the core identification accuracy across various core models and varying core-periphery degree gaps.

%% file: data.tex
\section{Core extraction in a statistics citation network}\label{sec:data}
We illustrate the impact of our core extraction method in downstream community analysis for the paper citation network collected by \citet{ji2016coauthorship}. Each node of the network is a paper, and two nodes are connected if one paper cited the other (ignoring the citation direction).  We focus on the largest connected component of the network. This network has $2248$ nodes and the average node degree is $4.95$. In \Cref{fig:statcite_plot}, we plot the whole citation network, and the core component extracted by \Cref{alg:cp_detection1} and \Cref{alg:cp_detection2}, with two different core sizes. The core sizes are selected to match that of the $k$-core algorithm, for easy comparison between the two approaches.

\begin{figure}
\centering
\begin{subfigure}{0.45\textwidth}
  \centering
  \includegraphics[width=1\textwidth]{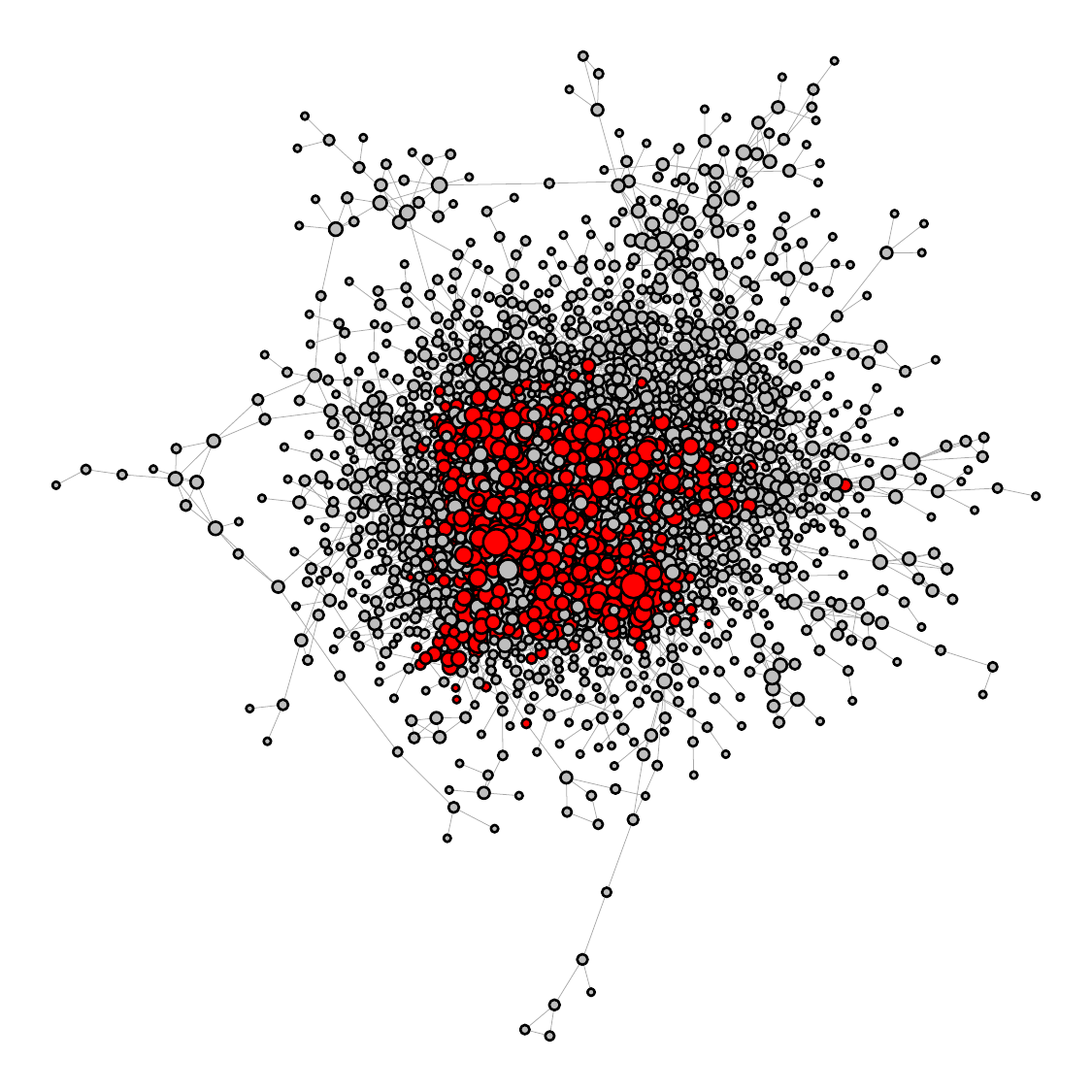}
  \caption{ER-type, $N_{\cal C} = 635$.}
  \label{fig:statcite_ER635}
\end{subfigure}
\begin{subfigure}{0.45\textwidth}
  \centering
  \includegraphics[width=1\textwidth]{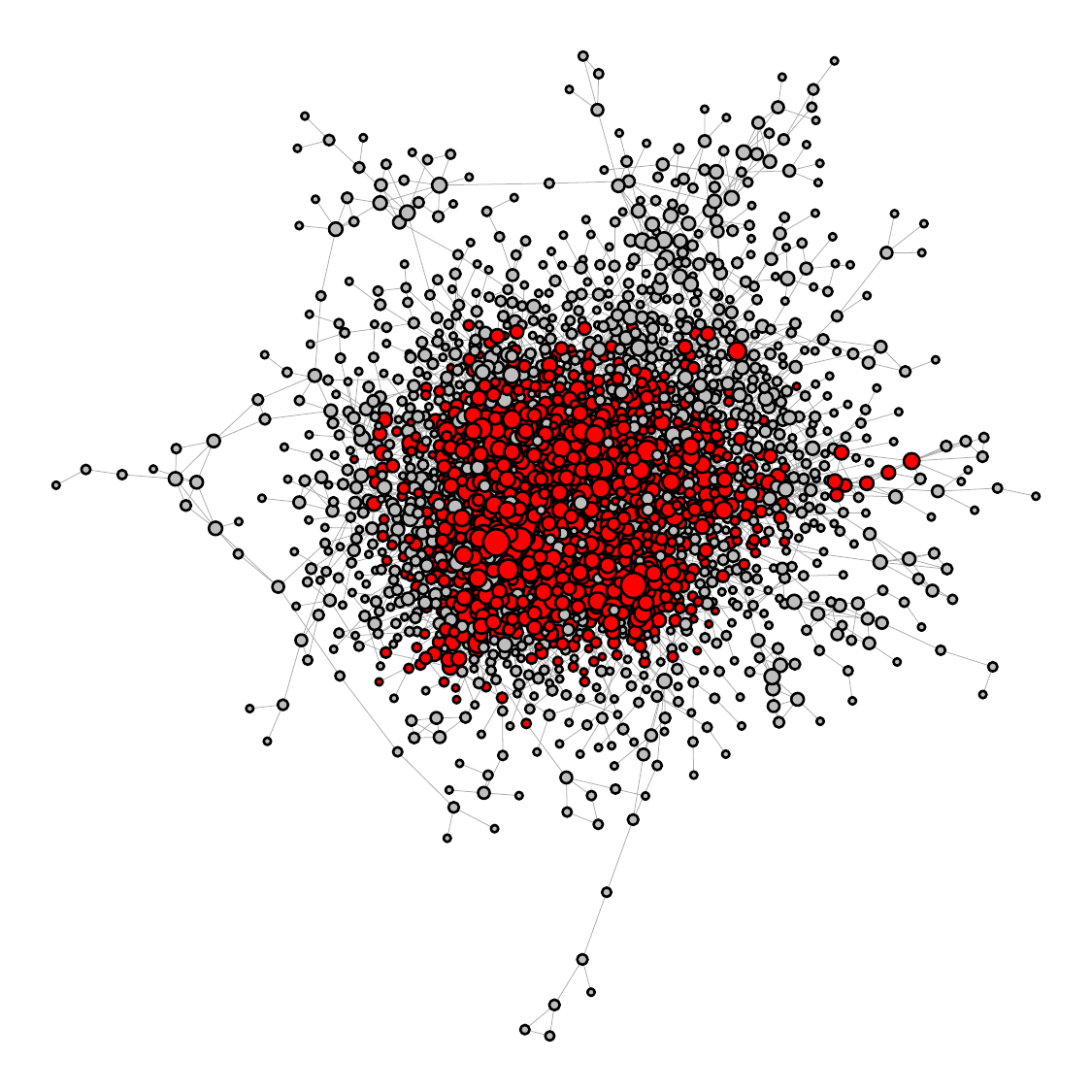}
  \caption{ER-type, $N_{\cal C} = 1103$.}
  \label{fig:statcite_ER1103}
\end{subfigure}
\begin{subfigure}{0.45\textwidth}
  \centering
  \includegraphics[width=1\textwidth]{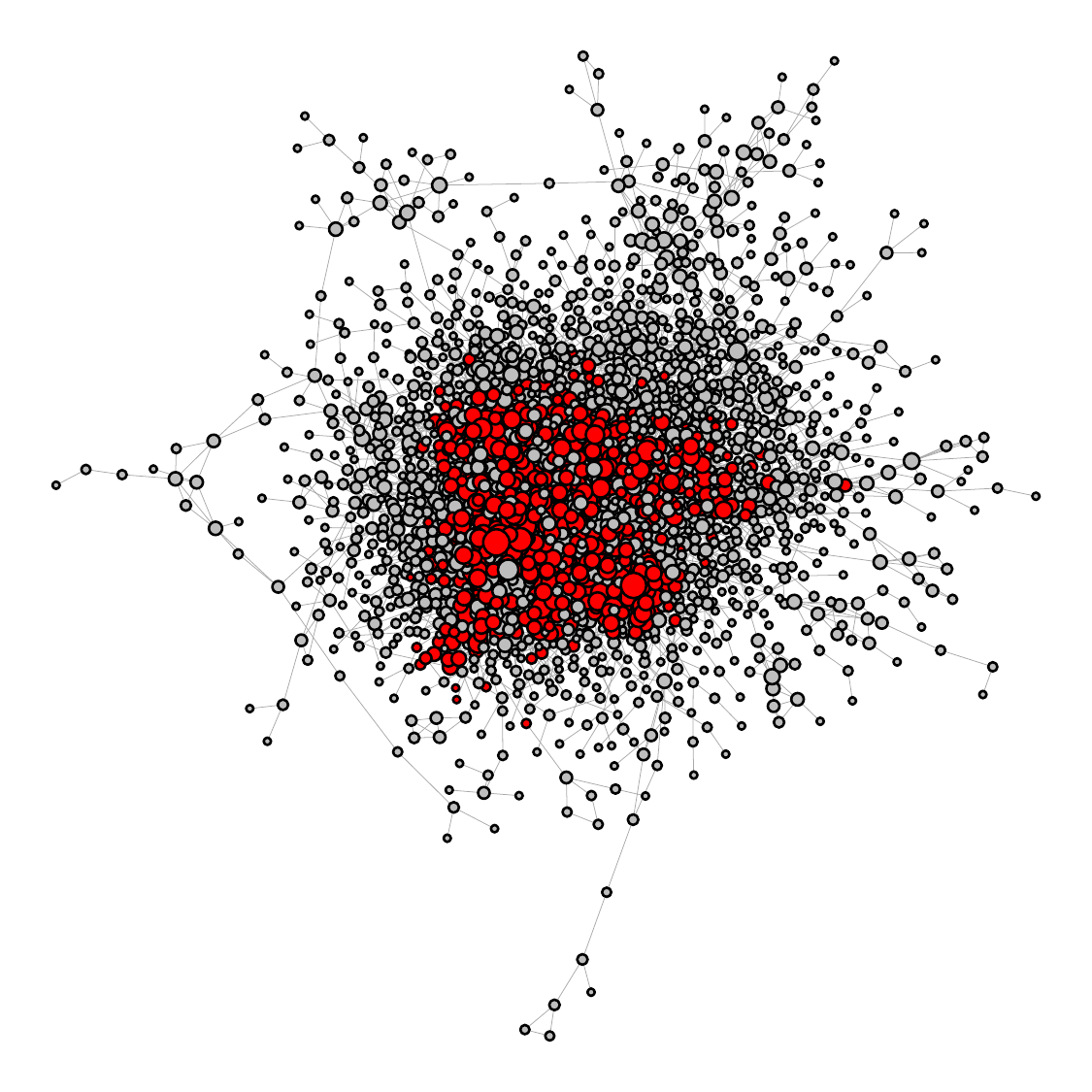}
  \caption{Configuration-type, $N_{\cal C} = 635$}
  \label{fig:statcite_Config635}
\end{subfigure}
\begin{subfigure}{0.45\textwidth}
  \centering
  \includegraphics[width=1\textwidth]{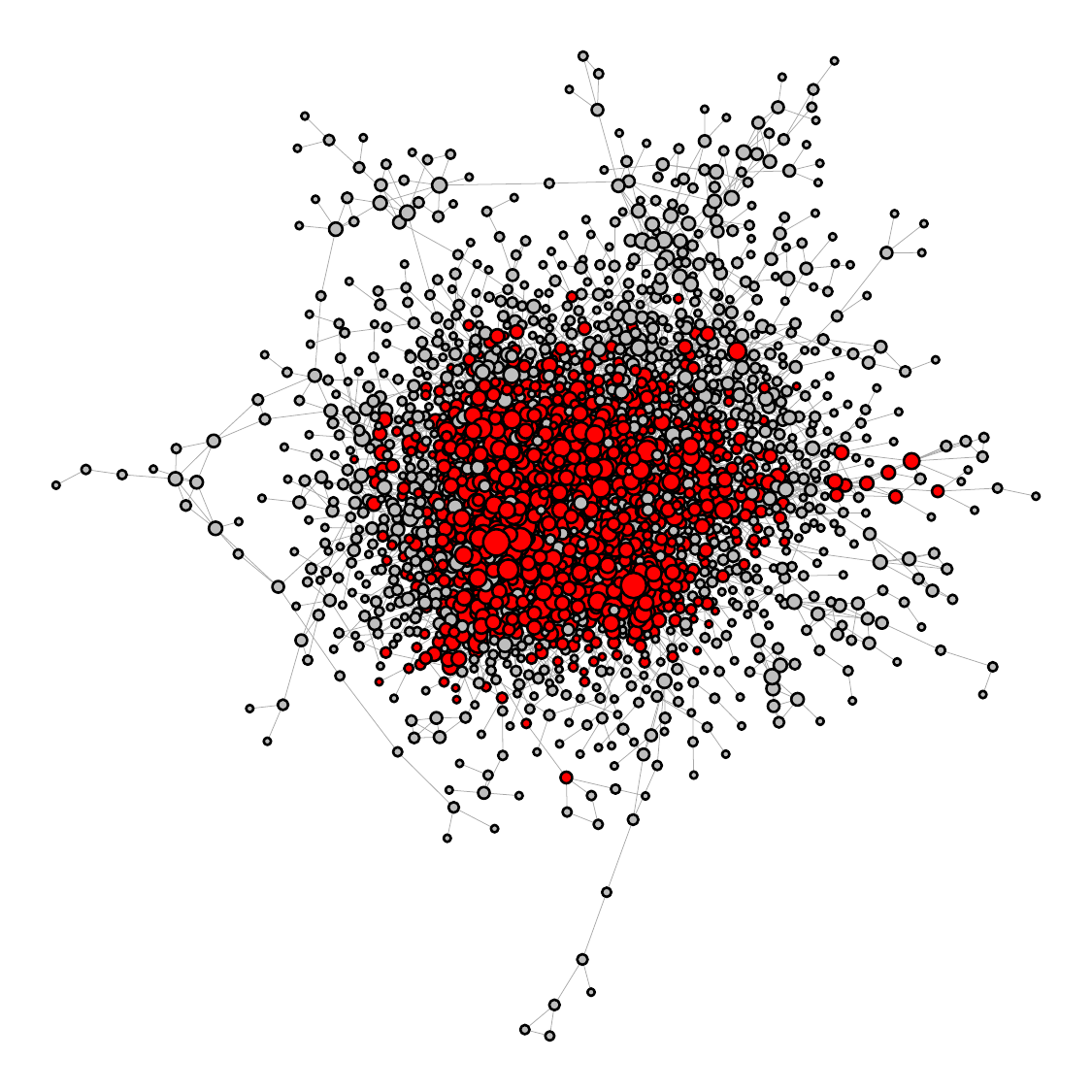}
  \caption{Configuration-type, $N_{\cal C} = 1103$}
  \label{fig:statcite_Config1103}
\end{subfigure}
\caption{Plots of the citation network, and the core components are highlighted in red.}
\label{fig:statcite_plot}
\end{figure}

In the analysis of \citet{wang2016discussion},  the $4$-core pruning is applied to the network, resulting in a core of $635$ nodes for their downstream analysis.  In this example, we compare several methods in Section~\ref{sec:simulation} and evaluate the performance by comparing the validity of the hierarchical community detection results on the extracted cores. For fair comparisons, we follow \citet{wang2016discussion} to use either  $3$-core and $4$-core pruning algorithms to obtain cores of size 1103 and 635, respectively. We then use other algorithms to extract cores of the same sizes. In addition to our methods, the other benchmark methods applicable for this task include degree centrality, eigenvector centrality, PageRank centrality, and local clustering coefficient. 

The hierarchical community detection (HCD) algorithm from \citet{li2020hierarchical} is then applied to the extracted cores. The HCD simultaneously detects the community membership and the hierarchical relation between the communities in the form of a binary tree. According to \citet{li2020hierarchical}, this hierarchical relationship can be transformed into a similarity matrix $\bm{S}$ where $\bm{S}_{kk'}$ measures the similarity between community $k$ and $k'$ along the hierarchy. 

We want to evaluate the meaningfulness of the hierarchical relationships in a quantitative way by comparing the hierarchical similarity $\bm{S}$ (based on the citation network) with the content similarity based on text data. In particular, the abstracts of all papers are available from \citet{wang2016discussion}. We represent each abstract as a term-frequency vector and apply the standard text mining processing such as stemming and stop words (including punctuations and numbers) removing. The term frequency-inverse document frequency (TF- IDF) weighting \citep{rajaraman2011mining} is then applied to each word. We remove words that appear in less than $1\%$ of the papers, and $966$ words remain after processing. The cosine similarity between each pair of papers is calculated, and a community level similarity matrix $\bm{T}$ is constructed where  $\bm{T}_{kk'}$ is the average cosine similarity between papers from community $k$ and community $k'$. We then calculate the Spearman correlation between $\bm{S}$ and $\bm{T}$ as a metric to measure how well the hierarchical structure discovered by HCD from the network matches the similarity derived from the abstracts. The results for cores extracted by different methods are summarized in \Cref{tab:sim_mat_cor}.

\begin{table}
\caption{Correlation between $\bm{S}$ and $\bm{T}$.}
\label{tab:sim_mat_cor}
\centering
\begin{tabular}{c | c | c }
\multirow{2}{*}{Methods} &
      \multicolumn{2}{c}{ Correlation } \\
    	& $N_{\cal C} = 635$ 	& $N_{\cal C} = 1103$ \\
\hline
\hline
Degree	& $0.099$ 	& $0.089$ \\
$k$-core	& $0.167$ 	& $0.108$ \\
PageRank	& $0.013$ 	& $0.106$ \\
EigenVec	& $0.143$ 	& $0.050$ \\
Local CC	& $0.058$ 	& $0.045$ \\
Ours (ER)	& $\bm{0.340}$ 	& $\bm{0.164}$ \\
Ours (Config)	& $\bm{0.350}$ 	& $\bm{0.155}$ \\
\end{tabular}
\end{table}

It can be seen that the cores extracted by both of our two models render significantly more meaningful hierarchies than the other benchmarks. The difference between the ER-type model and the configuration-type model is negligible. Also, applying HCD to the two cores from the ER-type model and the configuration-type model leads to the same hierarchical structure, with some marginal differences.

\Cref{fig:statcite_config_hierarchy} shows the extracted core by the configuration-type model with $N_{\cal C}=635$, and the corresponding hierarchical structure given by the HCD algorithm. It turns out that the community labels are also very interpretable. Since each cluster is a group papers, we list the most frequent keywords of the papers in each cluster in \Cref{tab:keywords}. The keywords in each group are highly coherent. 

\begin{figure}
\centering
  \includegraphics[width=0.8\textwidth]{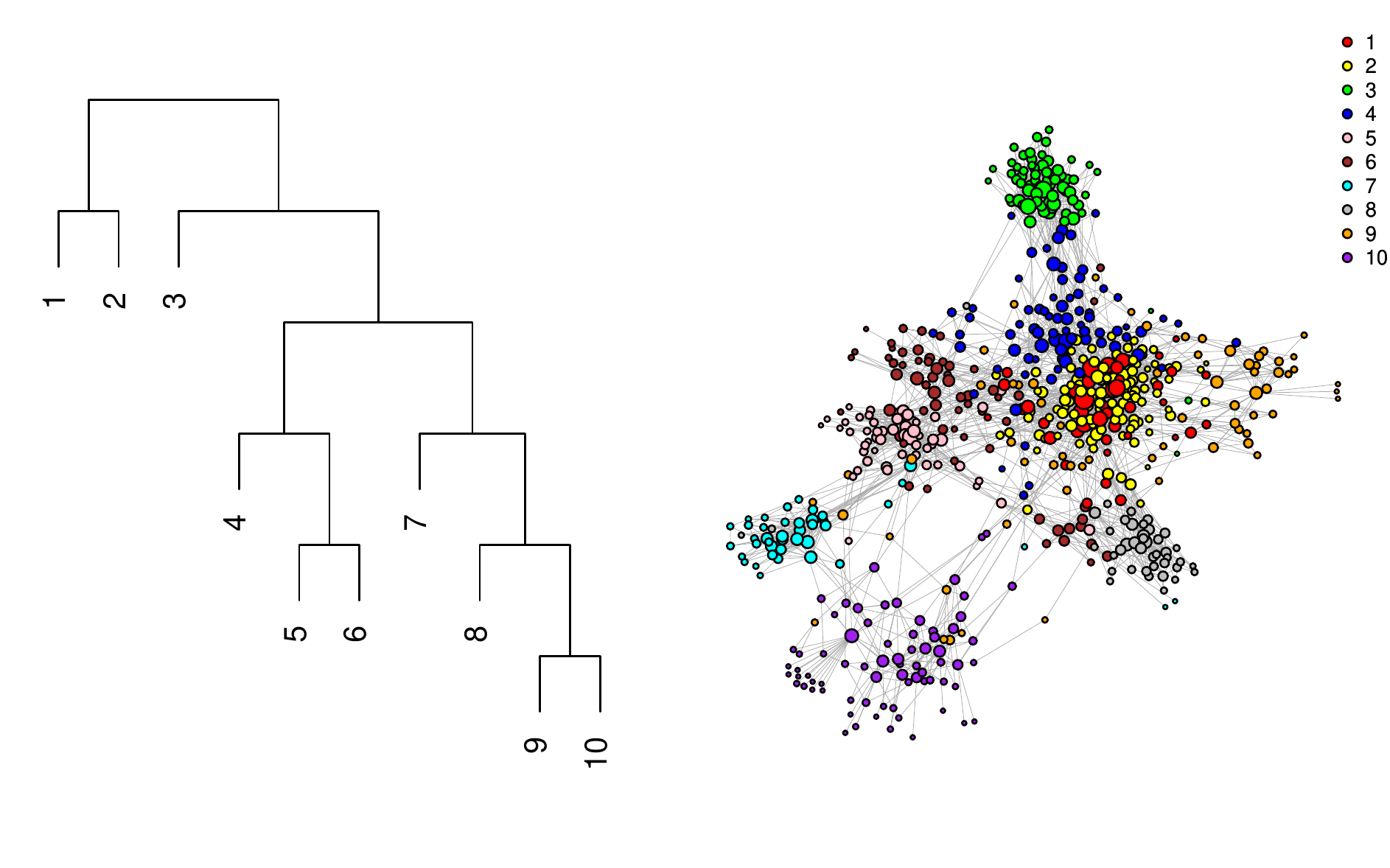}
  \caption{Hierarchical community structure of the core. The core has $N_{\cal C}=635$, and the configuration-type model is used.}
  \label{fig:statcite_config_hierarchy}
\end{figure}

\begin{table}\small
\caption{The most frequent keywords for each cluster in the hierachy.}
\label{tab:keywords}
\centering
\begin{tabular}{c | c}
Cluster & Most frequent keywords \\
\hline
\hline
1 & \makecell{ lasso, variable selection, smoothly clipped absolute deviation,\\ model selection, asymptotic normality, sparsity } \\ 
\hline
2 & \makecell{ lasso, variable selection, oracle property, sparsity,\\ regularization, model selection, smoothly clipped absolute deviation } \\ 
\hline
3 & \makecell{ false discovery rate, multiple testing, multiple comparisons,\\ familywise error rate, p-value, stepdown procedure} \\ 
\hline
4 & \makecell{ sparsity, lasso, regularization, covariance matrix,\\ high dimensional data, model selection, thresholding } \\ 
\hline
5 & \makecell{ functional data, smoothing, principal component,\\ eigenfunction, eigenvalue, functional regression } \\ 
\hline
6 & \makecell{ nonparametric regression, generalized estimating equation, functional data,\\ longitudinal data, partially linear model, semiparametric model } \\ 
\hline
7 & \makecell{ mixture model, nonparametric bayes, dirichlet process,\\ hierarchical model, stick breaking } \\ 
\hline
8 & \makecell{ sliced inverse regression, central subspace, sliced average variance estimation,\\ dimension reduction, nonparametric regression } \\ 
\hline
9 & \makecell{ classification, model selection, oracle inequality,\\ support vector machine, aggregation, sparsity, statistical learning } \\ 
\hline
10 & \makecell{ markov chain monte carlo, bayesian inference, gaussian markov random field,\\ gaussian process, generalized linear mixed model, kriging, spatial statistics } \\
\end{tabular}
\end{table}

%% file: appendix.tex
\newpage

\begin{appendix}

\section{ Proofs }

\subsection{ Proofs under the ER-type model}\label{appendix:proof_strong_consistency_1}

Let $\bm{U}$, $\bm{\Lambda}$, $\hat{\bm{U}}$, $\hat{\bm{\Lambda}}$ be defined as in \eqref{eq:svd_P} and \eqref{eq:svd_A}. We introduce the following additional notations to be used:
\begin{itemize}
\item $p^* = \max_{1 \le i,j \le n} P_{ij}$.
\item $\Delta = |\lambda_r| - |\lambda_{r+1}|$.
\item $\kappa = \min\{| \lambda_1 / \lambda_r |, 2r\}$.
\item $R = (\gamma+1)\log{n} + r$, where $\gamma > 0$.
\item $g=\sqrt{d_{\max}} + \frac{R}{\alpha\log{R}}$, where $\alpha \in (0,1)$.
\item $B(r) = 10 \min\{r, 1+\log_2(|\lambda_1/\lambda_r|)\}$.
\item $d_{\min} = \min_{1 \le i \le n} \sum_{j=1}^{n}P_{ij}$.
\item $d_{\max} = \max_{1 \le i \le n} \sum_{j=1}^{n}P_{ij}$.
\end{itemize}
As preparation, the following lemmas will be used in our proofs. 
\begin{lemma}[\citet{lei2019unified}]\label{lemma:Lei2019_svd_bound}
If $\Delta \succeq \kappa g $ and $| \lambda_{r} | \succeq  \frac{np^*}{\sqrt{n}\norm{ \bm{U} }_{2,\infty}} $, we have
\begin{multline}\label{eq:lemma_Lei2019_svd_bound}
\norm{ \bm{U} \bm{\Lambda} \bm{U}^t - \hat{\bm{U}} \hat{\bm{\Lambda}} \hat{\bm{U}}^t }_{2,\infty}
\preceq \sqrt{n} \left( \frac{\kappa g}{\Delta} \norm{ \bm{U} }_{2,\infty} + \frac{\sqrt{R p^*}}{|\lambda_r|} \right) \norm{ \bm{U} }_{2,\infty} |\lambda_1| \\
+ \sqrt{n} \norm{ \bm{U} }_{2,\infty}^2 \norm{ \bm{E} }_2 + \sqrt{n} {\left( \frac{\kappa g}{\Delta} \norm{ \bm{U} }_{2,\infty} + \frac{\sqrt{R p^*}}{|\lambda_r|} \right)}^2 \left( |\lambda_1| + \norm{ \bm{E} }_2 \right) ,
\end{multline}
with probability $1-(B(r)+1)n^{-\gamma}$.
\end{lemma}
\begin{proof}
This lemma can be proved by combining the Corollary 3.6 and the result in Section 7.4 from \citet{lei2019unified}.
\end{proof}

\begin{lemma}[Theorem 5.2 of \citet{lei2015consistency}]\label{thm:lei_rinaldo}
For $c_0 > 0$ and $\gamma > 0$ there exists a constant $C = C(\gamma, c_0)$ such that
\begin{equation}\label{eq:lemma_lei_rinaldo}
\norm{ \bm{E} }_2 \le C\max\{ \sqrt{n p^*}, \sqrt{c_0 \log{n}}\}
\end{equation}
with probability at least $1-n^{-\gamma}$.
\end{lemma}

Combining \Cref{lemma:Lei2019_svd_bound} and \Cref{thm:lei_rinaldo} would leads to a concentration bound of low-rank approximation with respect to the $\norm{\cdot}_{2,\infty}$.
\begin{lemma}\label{lemma:combined}
Suppose $n p^* \succeq \log{n}$, $\Delta \succeq \kappa g $, and $| \lambda_{r} | \succeq  \frac{np^*}{\sqrt{n}\norm{ \bm{U} }_{2,\infty}} $. Then, with probability at least $1-(B(r)+2) n^{-\gamma}$, we have
\begin{multline}\label{eq:lemma_combined_eq}
\norm{ \bm{U} \bm{\Lambda} \bm{U}^t - \hat{\bm{U}} \hat{\bm{\Lambda}} \hat{\bm{U}}^t }_{2,\infty} \preceq \sqrt{n} \bigg( \frac{\kappa g}{\Delta} \norm{ \bm{U} }_{2,\infty}^2 |\lambda_1| + \frac{\sqrt{R p^*}}{|\lambda_r|} \norm{ \bm{U} }_{2,\infty} |\lambda_1| \\
+ \norm{ \bm{U} }_{2,\infty}^2 \sqrt{n p^*} + {\left(\frac{\kappa g}{\Delta} \right)}^2 \norm{ \bm{U} }_{2,\infty}^2 (|\lambda_1| + \sqrt{np^*}) + \frac{R p^*}{\lambda_r^2} (|\lambda_1| + \sqrt{np^*}) \bigg).
\end{multline}
Furthermore, if \Cref{assumption:incoherence} holds, \eqref{eq:lemma_combined_eq} becomes
\begin{multline}\label{eq:lemma_combined_incoherent_eq}
\norm{ \bm{U} \bm{\Lambda} \bm{U}^t - \hat{\bm{U}} \hat{\bm{\Lambda}} \hat{\bm{U}}^t }_{2,\infty} \preceq \frac{\mu_0^2 r \kappa g}{\sqrt{n} \Delta} |\lambda_1| + \mu_0 |\lambda_1| \frac{\sqrt{r R p^*}}{|\lambda_r|} + \mu_0^2 r \sqrt{p^*} \\
+ {\left(\frac{\kappa g}{\Delta} \right)}^2 \mu_0^2 \frac{r}{\sqrt{n}} (|\lambda_1| + \sqrt{np^*}) + \frac{R p^*}{\lambda_r^2} (\sqrt{n}|\lambda_1| + n\sqrt{p^*}) .
\end{multline}
\end{lemma}
\begin{proof}
Plugging \eqref{eq:lemma_lei_rinaldo} into \eqref{eq:lemma_Lei2019_svd_bound}, and applying union bound, we get \eqref{eq:lemma_combined_eq}. Plugging $\norm{\bm{U}}_{2,\infty} \le \mu_0\sqrt{\frac{r}{n}}$ into \eqref{eq:lemma_combined_eq}, we get \eqref{eq:lemma_combined_incoherent_eq}.

\end{proof}

We now introduce the following theorem that includes \Cref{cor:strong_consistency_1} as a special case. 
\begin{theorem}\label{theorem:strong_consistency_1}
Assume the network $\bm{A}$ is generated from the ER-type model in \Cref{def:core_periphery1} under \Cref{assumption:incoherence}. Suppose $\Delta \succeq \kappa g $, $| \lambda_{r} | \succeq \frac{np^*}{\sqrt{n}\norm{ \bm{U} }_{2,\infty}} $, and $n p^* \succeq \log{n}$. Furthermore, if we have
\begin{multline}\label{eq:strong_consistency_1}
h(n) \ge C \bigg[
\frac{\mu_0^2 r}{\sqrt{n}} |\lambda_1| \left( \frac{ \kappa g}{\Delta} + \frac{ R}{|\lambda_r|} \right) + \mu_0 |\lambda_1| \frac{\sqrt{r p^* R }}{|\lambda_r|} + \mu_0^2 r \sqrt{p^*} \\
+ {\left(\frac{\kappa g}{\Delta} + \frac{R}{|\lambda_r|}\right)}^2 \frac{\mu_0^2 r}{\sqrt{n}} (|\lambda_1| + \sqrt{np^*}) + \frac{ p^* R \sqrt{n} }{\lambda_r^2} (|\lambda_1| + \sqrt{n p^*}) \bigg] + 2|\lambda_{r+1}| + p^*,
\end{multline}
then, for sufficiently large $n$, \Cref{alg:cp_detection1} exactly identifies the core and periphery set with probability $1-(B(r)+2) n^{-\gamma}$.
\end{theorem}

\begin{proof}[Proof of \Cref{theorem:strong_consistency_1} and \Cref{cor:strong_consistency_1}]
To achieve an exact separation between the core and periphery, we need 
\begin{equation}\label{eq:strong_consistency_condition1}
\min_{i \in {\cal C}} \norm{ \hat{\bm{P}}_{i,*} \bm{H} }_{2} > \max_{i \in {\cal P}} \norm{ \hat{\bm{P}}_{i,*} \bm{H} }_{2}.
\end{equation}
By triangular inequality, we have the following:
$$\text{For $i \in {\cal C}$: } \norm{ \hat{\bm{P}}_{i,*} \bm{H} }_{2} \ge \norm{ \bm{P}_{i,*} \bm{H} }_{2} - \norm{ \bm{P}_{i,*} \bm{H} - \hat{\bm{P}}_{i,*} \bm{H} }_{2} \ge h(n) - \norm{ \bm{P} \bm{H} - \hat{\bm{P}} \bm{H} }_{2,\infty}. $$
$$\text{For $i \in {\cal P}$: } \norm{ \hat{\bm{P}}_{i,*} \bm{H} }_{2} \le \norm{ \bm{P}_{i,*} \bm{H} }_{2} + \norm{ \bm{P}_{i,*} \bm{H} - \hat{\bm{P}}_{i,*} \bm{H} }_{2} \le p^* + \norm{ \bm{P} \bm{H} - \hat{\bm{P}} \bm{H} }_{2,\infty}. $$
Therefore, to satisfy $\eqref{eq:strong_consistency_condition1}$, it is thus sufficient to ensure that
\begin{equation}\label{eq:strong_consistency_condition2}
\norm{ \bm{P} \bm{H} - \hat{\bm{P}} \bm{H} }_{2,\infty} \le \frac{1}{2}( h(n) - p^* ).
\end{equation}
By the basic properties of $\norm{\cdot}_{2,\infty}$, we have
\begin{multline}\label{eq:strong_consistency_condition3}
\norm{ \bm{P} \bm{H} - \hat{\bm{P}} \bm{H} }_{2,\infty} \le \norm{ \bm{P} - \hat{\bm{P}} }_{2,\infty} \norm{ \bm{H} }_{2} = \norm{ \bm{P} - \hat{\bm{P}} }_{2,\infty} \\
= \norm{ \bm{U} \bm{\Lambda} \bm{U}^t + \bm{U}_{\bot} \bm{\Lambda}_{\bot} \bm{U}_{\bot}^t - \hat{\bm{U}} \hat{\bm{\Lambda}} \hat{\bm{U}}^t }_{2,\infty}
\le \norm{ \bm{U} \bm{\Lambda} \bm{U}^t - \hat{\bm{U}} \hat{\bm{\Lambda}} \hat{\bm{U}}^t }_{2,\infty} + \norm{ \bm{U}_{\bot} \bm{\Lambda}_{\bot} \bm{U}_{\bot}^t }_{2,\infty} \\
\le \norm{ \bm{U} \bm{\Lambda} \bm{U}^t - \hat{\bm{U}} \hat{\bm{\Lambda}} \hat{\bm{U}}^t }_{2,\infty} + \norm{ \bm{U}_{\bot} \bm{\Lambda}_{\bot} \bm{U}_{\bot}^t }_{2} = \norm{ \bm{U} \bm{\Lambda} \bm{U}^t - \hat{\bm{U}} \hat{\bm{\Lambda}} \hat{\bm{U}}^t }_{2,\infty} + | \lambda_{r+1} | .
\end{multline}
By \eqref{eq:strong_consistency_condition3} and \eqref{eq:lemma_combined_eq} of \Cref{lemma:combined}, \eqref{eq:strong_consistency_condition2} holds with probability at least $1-(B(r)+2) n^{-\gamma}$ as long as

\begin{multline*}
h(n) \ge C \bigg[ \frac{\mu_0^2 r \kappa g}{\sqrt{n} \Delta} |\lambda_1| + \mu_0 |\lambda_1| \frac{\sqrt{r R p^*}}{|\lambda_r|} + \mu_0^2 r \sqrt{p^*} \\
+ {\left(\frac{\kappa g}{\Delta} \right)}^2 \mu_0^2 \frac{r}{\sqrt{n}} (|\lambda_1| + \sqrt{np^*}) + \frac{R p^*}{\lambda_r^2} (\sqrt{n}|\lambda_1| + n\sqrt{p^*}) \bigg] + 2|\lambda_{r+1}| + p^*.
\end{multline*}
as assumed by \Cref{theorem:strong_consistency_1}.

Now we proceed to prove \Cref{cor:strong_consistency_1}, under the additional conditions of $p^* \succeq \max\left\{ \frac{\mu_0^2 r \log{n} }{n}, \frac{\mu_0^2 r^2}{n} \right\} $, boundedness of $\left| \lambda_1 / \lambda_r \right|$ and \Cref{assumption:low_rank}. Combined with the fact that
$$g \preceq \sqrt{n p^*} + \log{n} + r,$$ 
these conditions lead to
$$|\lambda_r| \succ |\lambda_{r+1}|$$
and
 $$\Delta \simeq |\lambda_r| \succeq \frac{n p^*}{\mu_0 \sqrt{r}} \succeq g.$$ 
 Therefore the conditions of \Cref{eq:lemma_combined_eq} hold.  Inserting the result of  \Cref{lemma:combined_succinct} into the third step of \eqref{eq:strong_consistency_condition3} leads to
 $$\norm{ \bm{P} \bm{H} - \hat{\bm{P}} \bm{H} }_{2,\infty}  \preceq \mu_0 \sqrt{ r (\log{n} + r) p^* } + \mu_0^2 r \sqrt{ p^* }$$
 
Therefore, \eqref{eq:strong_consistency_condition2} is satisfied if
\begin{equation*}
h(n) \succeq \mu_0 \sqrt{ r (\log{n} + r) p^* } + \mu_0^2 r \sqrt{ p^* } .
\end{equation*}
as assumed in \Cref{cor:strong_consistency_1}.
\end{proof}

\begin{lemma}\label{lemma:combined_succinct}
Under the same conditions in \Cref{lemma:combined}, suppose \Cref{assumption:low_rank} and \Cref{assumption:incoherence} hold. If $p^* \succeq \max\left\{ \frac{\mu_0^2 r \log{n} }{n}, \frac{\mu_0^2 r^2}{n} \right\} $, and $\left| \lambda_1 / \lambda_r \right|$ is bounded, we have
\begin{equation}\label{eq:bound0}
\norm{ \hat{\bm{P}} - \bm{P} }_{2,\infty} \preceq \mu_0 \sqrt{ r (\log{n} + r) p^* } + \mu_0^2 r \sqrt{ p^* },
\end{equation}
with probability at least $1-(B(r)+2) n^{-\gamma}$.
\end{lemma}
\begin{proof}[Proof of \Cref{lemma:combined_succinct}]

Since $p^* \succeq \frac{\mu_0^2 r \log{n} }{n}\succ \frac{r\log{n}}{n^2}$,  \Cref{assumption:low_rank}  indicates that $|\lambda_r| \succ |\lambda_{r+1}|$. Together with the boundedness assumption of $\left| \lambda_1 / \lambda_r \right|$, we know that $\left| \lambda_1 / \Delta \right|$ is also bounded. In addition, $\kappa$ also becomes bounded.  \eqref{eq:lemma_combined_incoherent_eq} becomes
\begin{multline}\label{eq:bound3}
\norm{ \bm{U} \bm{\Lambda} \bm{U}^t - \hat{\bm{U}} \hat{\bm{\Lambda}} \hat{\bm{U}}^t }_{2,\infty} \preceq \frac{\mu_0^2 r g}{ \sqrt{n} } + \mu_0\sqrt{r R p^*} + \mu_0^2 r \sqrt{p^*} \\
+ {\left(\frac{ g }{\Delta} \right)}^2 \frac{ \mu_0^2 r }{\sqrt{n}} (|\lambda_1| + \sqrt{np^*}) + \frac{R p^*}{\lambda_r^2} (\sqrt{n}|\lambda_1| + n\sqrt{p^*}) .
\end{multline}
Note that $\norm{\bm{U}}_{2,\infty} \ge \sqrt{\frac{r}{n}}$, so we have $\mu_0 \ge 1$. Therefore, $ | \lambda_1 | \ge | \lambda_r | \ge \frac{n p^*}{\mu_0 \sqrt{r}} $. When $p^* \succeq \frac{\mu_0^2 r}{n} $, we have $|\lambda_1| \succeq \sqrt{n p^*}$, and \eqref{eq:bound3} becomes
\begin{equation}\label{eq:bound4}
\norm{ \bm{U} \bm{\Lambda} \bm{U}^t - \hat{\bm{U}} \hat{\bm{\Lambda}} \hat{\bm{U}}^t }_{2,\infty} \preceq \frac{\mu_0^2 r g}{ \sqrt{n} } + \mu_0\sqrt{r R p^*} + \mu_0^2 r \sqrt{p^*}
+ \frac{ g^2 }{\Delta} \frac{ \mu_0^2 r }{\sqrt{n}} + \frac{R p^*}{|\lambda_r|} \sqrt{n} .
\end{equation}
Furthermore, due to the assumption $r \le \sqrt{np^*}$, we have $g \preceq \Delta$ and $\frac{ g^2 }{\Delta} \frac{ \mu_0^2 r }{\sqrt{n}} \preceq  \mu_0^2 r \sqrt{p^*}$. So \eqref{eq:bound4} becomes
\begin{equation}\label{eq:bound5}
\norm{ \bm{U} \bm{\Lambda} \bm{U}^t - \hat{\bm{U}} \hat{\bm{\Lambda}} \hat{\bm{U}}^t }_{2,\infty} \preceq \frac{\mu_0^2 r g}{ \sqrt{n} } + \mu_0\sqrt{r R p^*} + \mu_0^2 r \sqrt{p^*} + \frac{R p^*}{|\lambda_r|} \sqrt{n} .
\end{equation}
Plugging in $|\lambda_{r}| \succeq \frac{n p^*}{\sqrt{r}} $, we get
\begin{multline}\label{eq:bound5}
\norm{ \bm{U} \bm{\Lambda} \bm{U}^t - \hat{\bm{U}} \hat{\bm{\Lambda}} \hat{\bm{U}}^t }_{2,\infty} \preceq \frac{\mu_0^2 r g}{ \sqrt{n} } + \mu_0\sqrt{r R p^*} + \mu_0^2 r \sqrt{p^*} + \frac{\sqrt{r}R}{\sqrt{n}} \\ 
\preceq \frac{\mu_0^2 r (\sqrt{np^*} + \log{n} + r)}{ \sqrt{n} } + \mu_0\sqrt{r (\log{n}+r) p^*} + \mu_0^2 r \sqrt{p^*} .
\end{multline}
Taking into account the condition $p^* \succeq \max\left\{ \frac{\mu_0^2 r \log{n} }{n}, \frac{\mu_0^2 r^2}{n} \right\} $, \eqref{eq:bound5} becomes
\begin{equation}\label{eq:bound6}
\norm{ \bm{U} \bm{\Lambda} \bm{U}^t - \hat{\bm{U}} \hat{\bm{\Lambda}} \hat{\bm{U}}^t }_{2,\infty} \preceq \mu_0 \sqrt{ r (\log{n} + r) p^* } + \mu_0^2 r \sqrt{ p^* } .
\end{equation}
Finally, we have
\begin{align*}
\norm{ \bm{P} - \hat{\bm{P}} }_{2,\infty} &\le \norm{ \bm{U} \bm{\Lambda} \bm{U}^t - \hat{\bm{U}} \hat{\bm{\Lambda}} \hat{\bm{U}}^t }_{2,\infty} + \norm{\bm{U}_{\bot} \bm{\Lambda}_{\bot} \bm{U}_{\bot}^t}_{2,\infty} \\
& \le \norm{ \bm{U} \bm{\Lambda} \bm{U}^t - \hat{\bm{U}} \hat{\bm{\Lambda}} \hat{\bm{U}}^t }_{2,\infty} + \norm{\bm{U}_{\bot} \bm{\Lambda}_{\bot} \bm{U}_{\bot}^t}_{2} \\
&\le \norm{ \bm{U} \bm{\Lambda} \bm{U}^t - \hat{\bm{U}} \hat{\bm{\Lambda}} \hat{\bm{U}}^t }_{2,\infty} + | \lambda_{r+1} | \\
&\preceq \mu_0 \sqrt{ r (\log{n} + r) p^* } + \mu_0^2 r \sqrt{ p^* } ,
\end{align*}
with probability at least $1-(B(r)+2)n^{-\gamma}$.
\end{proof}

\Cref{prop:threshold_1} can be proved by specifying a data-driven estimate of the separation order between the scores of the core and periphery components.

\begin{proof}[Proof of \Cref{prop:threshold_1}]
For any $i \in {\cal P}$, we have $\norm{\bm{P}_{i,*}\bm{H}} < p^*$. Under the event of \Cref{lemma:combined_succinct}, by the boundedness of $\mu_0$ and $r$, we have
\begin{align*}
S_i &= \norm{\hat{\bm{P}}_{i,*}\bm{H}}_2 \\
&\le \norm{\bm{P}_{i,*}\bm{H}}_2 + \norm{\hat{\bm{P}}_{i,*}\bm{H} - \bm{P}_{i,*}\bm{H}}_2 \\
&< p^* + \norm{ \hat{\bm{P}}_{i,*} - \bm{P}_{i,*} }_2 \\
&\preceq \sqrt{p^* \log{n}};
\end{align*}
Similarly for any $i \in {\cal C}$, since $\norm{\bm{P}_{i,*}\bm{H}} \ge h(n)$, we have
\begin{align*}
S_i &= \norm{\hat{\bm{P}}_{i,*}\bm{H}}_2 \\
&\ge h(n) - \norm{\hat{\bm{P}}_{i,*}\bm{H} - \bm{P}_{i,*}\bm{H}}_2 \\
&\succeq h(n) - \sqrt{p^* \log{n}}.
\end{align*}

Recall  that $\hat{p} = \frac{2}{n^2-n}\sum_{i<j} \bm{A}_{ij}$ and 
$\min_{1 \le i,j \le n} \bm{P}_{ij} \simeq \max_{1 \le i,j \le n} \bm{P}_{ij} = p^*$. By Hoeffding's inequality, we know that $\hat{p} \simeq p^*$ with probability greater than $1-2n^{-\gamma}$. Therefore, if $h(n) \succ \sqrt{p^{*(1-\epsilon)} \log{n}}$, for some small constant $\epsilon$, $ S_i \succ \sqrt{\hat{p}^{1-\epsilon} \log{n}} $ for $i \in {\cal C}$, and $S_i \preceq \sqrt{\hat{p} \log{n}}$ for $i \in {\cal P}$, for sufficiently large $n$ with probability greater than $1-(B(r)+4)n^{-\gamma}$.

\end{proof}

Finally, the weak consistency can be proved by concentration results with respect to Frobenius norm.
\begin{proof}[Proof of \Cref{thm:weak_consistency_1}]
First, we want to bound $\norm{ \bm{P}\bm{H}-\hat{\bm{P}}\bm{H} }_F^2$.
\begin{align*}
\norm{ \bm{P}\bm{H}-\hat{\bm{P}}\bm{H} }_F^2 &\preceq \rank(\bm{P}\bm{H}-\hat{\bm{P}}\bm{H}) \cdot \norm{ \bm{P}\bm{H}-\hat{\bm{P}}\bm{H} }_2^2 \\
&\preceq \max\{ \rank(\bm{P}\bm{H}), \rank(\hat{\bm{P}}\bm{H}) \} \cdot \norm{ \bm{P}-\hat{\bm{P}} }_2^2 \\
&\le \max\{ \rank(\bm{P}), \rank(\hat{\bm{P}}) \} \cdot \norm{ \bm{P}-\hat{\bm{P}} }_2^2 \\
&= \max\{ \rank(\bm{P}), r \} \cdot \norm{ \bm{P}-\hat{\bm{P}} }_2^2 \\
&= \max\{ \rank(\bm{P}), r \} \cdot \norm{ \bm{P} - \bm{A} + \bm{A} -\hat{\bm{P}} }_2^2 \\
&\preceq \max\{ \rank(\bm{P}), r \} \cdot {\left( \norm{ \bm{P} - \bm{A} }_2 + \norm{ \bm{A} -\hat{\bm{P}} }_2 \right)}^2 \\
&\preceq \max\{ \rank(\bm{P}), r \} \cdot {\left( \norm{ \bm{P} - \bm{A} }_2 + | \hat{\lambda}_{r+1} | \right)}^2 \\
&\preceq \max\{ \rank(\bm{P}), r \} \cdot {\left( \norm{ \bm{P} - \bm{A} }_2 + \norm{ \bm{P} - \bm{A} }_2 + | \lambda_{r+1} | \right)}^2 \\
&\preceq \max\{ \rank(\bm{P}), r \} \cdot {\left( \norm{ \bm{P} - \bm{A} }_2 + | \lambda_{r+1} | \right)}^2.
\end{align*}

Each misclassification necessarily involves a squared deviation of order at least ${\left(h(n)-p^*\right)}^2$. Given the total squared deviation bounded by the above inequality, we can show that the number of misclassified nodes is at most
$$ M \preceq \frac{\norm{ \bm{P}\bm{H}-\hat{\bm{P}}\bm{H} }_F^2}{{\left(h(n) - p^*\right)}^2} = C \cdot \max\{r, \rank(\bm{P})\} \cdot \frac{ {\left( \norm{ \bm{P} - \bm{A} }_2 + | \lambda_{r+1} | \right)}^2 }{ {\left(h(n)-p^*\right)}^2 }, $$
where $C$ is some constant. Applying \Cref{thm:lei_rinaldo}, we can get
$$ M \preceq \max\{r, \rank(\bm{P})\} \cdot \frac{ {\left( \max\{\sqrt{n p^*},\sqrt{\log{n}}\} + | \lambda_{r+1} | \right)}^2 }{ {\left(h(n)-p^*\right)}^2 }, $$
with probability at least $1-n^{-\gamma}$.
\end{proof}

\subsection{ Proofs under the configuration-type model}\label{appendix:proof_strong_consistency_2}
We first introduce the ancillary lemmas:
\begin{lemma}[\citet{qin2013regularized}]\label{lemma:qin_rohe}
Suppose $d_{\min} > 3(\gamma + 1) \log{n}$. Then,
$$ \norm{ \hat{\bm{D}} \bm{D}^{-1} - \bm{I} }_{2} < \sqrt{ \frac{3(\gamma + 1) \log{n}}{d_{\min}} } $$
with probability greater than $1-\frac{2}{n^{\gamma}}$.
\end{lemma}
\begin{proof}
This lemma is indirectly proved in the proof of Theorem 4.1 in \citet{qin2013regularized}.
By setting $\tau = 0$, $\epsilon = 4n^{-\gamma}$, and $a = \sqrt{ \frac{3(\gamma+1)\log{n}}{d_{\min}} }$ in their proof, for each $i$, we can have
$$ P{\left( \abs{\frac{\hat{\bm{D}}_{ii}}{\bm{D}_{ii}}-1} \ge \sqrt{ \frac{3(\gamma+1)\log{n}}{d_{\min}} } \right)} \le 2n^{-\gamma-1} . $$
Then,
\begin{align*}
P{\left( \norm{ \hat{\bm{D}} \bm{D}^{-1} - \bm{I} }_{2} \ge \sqrt{ \frac{3(\gamma+1)\log{n}}{d_{\min}} } \right)} &= P{\left( \max_{i} \abs{\frac{\hat{\bm{D}}_{ii}}{\bm{D}_{ii}}-1} \ge \sqrt{ \frac{3(\gamma+1)\log{n}}{d_{\min}} } \right)} \\
&= P{\left( \cup_{i}\left\{ \abs{\frac{\hat{\bm{D}}_{ii}}{\bm{D}_{ii}}-1} \ge \sqrt{ \frac{3(\gamma+1)\log{n}}{d_{\min}} } \right\} \right)} \\
&\le \sum_{i=1}^{n} P{\left( \abs{\frac{\hat{\bm{D}}_{ii}}{\bm{D}_{ii}}-1} \ge \sqrt{ \frac{3(\gamma+1)\log{n}}{d_{\min}} } \right)} \\
&\le 2n^{-\gamma}.
\end{align*}
\end{proof}

\begin{lemma}\label{lemma:bound_periphery}
Under \Cref{def:core_periphery2}, we have $$\max_{i \in \pcal} \norm{ \bm{P}_{i,*} \bm{D}^{-1} \bm{H} }_{2} \le \frac{d_{\max}}{(n-1) d_{\min}}.$$
\end{lemma}
\begin{proof}
We assume the diagonal entries of $\bm{P}$ are $0$. By definition, for $i \in {\cal P}$ and $i \ne j$, $${[ \bm{P} \bm{D}^{-1} ]}_{ij} = \frac{d_i}{\sum_{k \ne i} d_k}.$$ So, $${[ \bm{P} \bm{D}^{-1} \bm{H} ]}_{ij} = \frac{d_i}{\sum_{k \ne i} d_k} - \frac{n-1}{n}\frac{d_i}{\sum_{k \ne i} d_k} = \frac{d_i}{n \sum_{k \ne i} d_k}$$ for $i \ne j$, and $$ {[ \bm{P} \bm{D}^{-1} \bm{H} ]}_{ii} = - \frac{n-1}{n}\frac{d_i}{\sum_{k \ne i} d_k}.$$
Therefore, we have
$$\norm{\bm{P}_{i,*} \bm{D}^{-1} \bm{H}}_{2} = \sqrt{(n-1){\left( \frac{1}{n} \right)}^2 + {\left(\frac{n-1}{n}\right)}^2 } \frac{d_i}{\sum_{k \ne i} d_k} = \sqrt{\frac{n-1}{n}} \frac{d_i}{\sum_{k \ne i} d_k} < \frac{d_{\max}}{(n-1) d_{\min}}.$$
\end{proof}

We give a more general theorem that includes \Cref{cor:strong_consistency_2} as a special case.
\begin{theorem}\label{theorem:strong_consistency_2}
Assume the network $\bm{A}$ is generated from the configuration-type model in \Cref{def:core_periphery2} under \Cref{assumption:incoherence}. Suppose $\Delta \succeq \kappa g $, $| \lambda_{r} | \succeq \frac{np^*}{\sqrt{n}\norm{ \bm{U} }_{2,\infty}} $, $d_{\min} \succ \log{n}$. If we have
\begin{multline}\label{eq:config_strong_condition_final}
h'(n) \succ \frac{\mu_0^2 r |\lambda_1|}{d_{\min} \sqrt{n}} \left( \frac{ \kappa g}{\Delta} + \frac{ R}{|\lambda_r|} \right) + \frac{ \mu_0 |\lambda_1| \sqrt{r p^* R}}{ d_{\min} |\lambda_r| } + \frac{\mu_0^2 r \sqrt{p^*}}{d_{\min}} \\
+ {\left(\frac{\kappa g}{\Delta} + \frac{R}{|\lambda_r|}\right)}^2 \frac{\mu_0^2 r}{d_{\min} \sqrt{n}} (|\lambda_1| + \sqrt{np^*}) + \frac{p^* R \sqrt{n}}{\lambda_r^2 d_{\min}} (|\lambda_1| + \sqrt{n p^*}) + \frac{|\lambda_{r+1}|}{d_{\min}} \\
+ \norm{\bm{P}\bm{D}^{-1}}_{2,\infty} \sqrt{\frac{3(\gamma + 1) \log{n}}{d_{\min}}} + \frac{d_{\max}}{n d_{\min}},
\end{multline}
then, \Cref{alg:cp_detection2} exactly identifies the core and periphery set with probability greater than $1-(B(r)+4)n^{-\gamma}$, for some positive constant $\gamma$.
\end{theorem}

\begin{proof}[Proof of \Cref{theorem:strong_consistency_2} and \Cref{cor:strong_consistency_2}]

First, we have $\norm{ \bm{P}_{i,*} \bm{D}^{-1} \bm{H} }_{2} \ge h'(n)$ for $i \in {\cal C}$. Also, by \Cref{lemma:bound_periphery}, we have that $\norm{ \bm{P}_{i,*} \bm{D}^{-1} \bm{H} }_{2} \le \frac{d_{\max}}{(n-1) d_{\min}}$ for $i \in {\cal P}$. To achieve strong consistency, we need to have 
\begin{equation}\label{eq:config_strong_condition_2}
h'(n) > 2 \norm{ \bm{P} \bm{D}^{-1} \bm{H} - \hat{\bm{P}} \hat{\bm{D}}^{-1} \bm{H}}_{2,\infty} + \frac{d_{\max}}{(n-1) d_{\min}}.
\end{equation}
In the following, we give a bound for $\norm{ \bm{P} \bm{D}^{-1} \bm{H} - \hat{\bm{P}} \hat{\bm{D}}^{-1} \bm{H} }_{2,\infty}$. Notice that
\begin{equation}\label{eq:config_bound_1}
\norm{ \bm{P} \bm{D}^{-1} \bm{H} - \hat{\bm{P}} \hat{\bm{D}}^{-1} \bm{H} }_{2,\infty}
\le \norm{ \bm{P} \bm{D}^{-1} - \hat{\bm{P}} \hat{\bm{D}}^{-1} }_{2,\infty} \norm{ \bm{H} }_2
\le \norm{ \bm{P} \bm{D}^{-1} - \hat{\bm{P}} \hat{\bm{D}}^{-1} }_{2,\infty}.
\end{equation}

Meanwhile, we have
\begin{multline*}\label{eq:config_bound_12}
\norm{ \bm{P} \bm{D}^{-1} - \hat{\bm{P}} \hat{\bm{D}}^{-1} }_{2,\infty}
 = \norm{ \bm{P} \bm{D}^{-1} - \hat{\bm{P}}\bm{D}^{-1} + \hat{\bm{P}}\bm{D}^{-1} - \hat{\bm{P}} \hat{\bm{D}}^{-1} }_{2,\infty} \\
 = \norm{ (\bm{P} - \hat{\bm{P}}) \bm{D}^{-1} + \hat{\bm{P}} \hat{\bm{D}}^{-1} (\hat{\bm{D}} \bm{D}^{-1} - \bm{I}) }_{2,\infty} \\
 = \norm{ (\bm{P} - \hat{\bm{P}}) \bm{D}^{-1} + ( \hat{\bm{P}} \hat{\bm{D}}^{-1} - \bm{P}\bm{D}^{-1} + \bm{P}\bm{D}^{-1} ) (\hat{\bm{D}} \bm{D}^{-1} - \bm{I}) }_{2,\infty} \\
 \le \norm{ \bm{P} - \hat{\bm{P}} }_{2,\infty} \norm{ \bm{D}^{-1} }_2 + \norm{\bm{P}\bm{D}^{-1}}_{2,\infty} \norm{\hat{\bm{D}} \bm{D}^{-1} - \bm{I}}_2 + \norm{ \bm{P} \bm{D}^{-1} - \hat{\bm{P}} \hat{\bm{D}}^{-1} }_{2,\infty} \norm{\hat{\bm{D}} \bm{D}^{-1} - \bm{I}}_2.
\end{multline*}
Moving the term $\norm{ \bm{P} \bm{D}^{-1} - \hat{\bm{P}} \hat{\bm{D}}^{-1} }_{2,\infty} \norm{\hat{\bm{D}} \bm{D}^{-1} - \bm{I}}_2$ from the right to the left, we get
\begin{equation}\label{eq:config_bound_2}
\left( 1 - \norm{\hat{\bm{D}} \bm{D}^{-1} - \bm{I}}_2 \right) \norm{ \bm{P} \bm{D}^{-1} - \hat{\bm{P}} \hat{\bm{D}}^{-1} }_{2,\infty}
\le \norm{ \bm{P} - \hat{\bm{P}} }_{2,\infty} \norm{ \bm{D}^{-1} }_2 + \norm{\bm{P}\bm{D}^{-1}}_{2,\infty} \norm{\hat{\bm{D}} \bm{D}^{-1} - \bm{I}}_2.
\end{equation}

By \Cref{lemma:qin_rohe}, if $d_{\min} \succ \log{n}$, $\norm{\hat{\bm{D}} \bm{D}^{-1} - \bm{I}}_2$ is vanishing for sufficiently large $n$ with high probability. Therefore, we have
\begin{multline}\label{eq:config_bound_3}
\norm{ \bm{P} \bm{D}^{-1} - \hat{\bm{P}} \hat{\bm{D}}^{-1} }_{2,\infty}
\preceq \norm{ \bm{P} - \hat{\bm{P}} }_{2,\infty} \norm{ \bm{D}^{-1} }_2 + \norm{\bm{P}\bm{D}^{-1}}_{2,\infty} \norm{\hat{\bm{D}} \bm{D}^{-1} - \bm{I}}_2 \\
\preceq \norm{ \bm{P} - \hat{\bm{P}} }_{2,\infty} \frac{1}{d_{\min}} + \norm{\bm{P}\bm{D}^{-1}}_{2,\infty} \sqrt{\frac{3(\gamma + 1) \log{n}}{d_{\min}}}
\end{multline}
with probability greater than $1-\frac{2}{n^{\gamma}}$.

Under \Cref{assumption:incoherence}, applying \Cref{lemma:combined} and \eqref{eq:config_bound_1}, we can see that \eqref{eq:config_strong_condition_2} are satisfied with probability greater than $1-(B(r)+4)n^{-\gamma}$, if
\begin{multline*}\label{eq:config_strong_condition_final}
h'(n) \succ \frac{\mu_0^2 r \kappa g |\lambda_1|}{ \Delta d_{\min} \sqrt{n} } + \frac{ \mu_0 |\lambda_1| \sqrt{r R p^*}}{|\lambda_r| d_{\min}} + \frac{\mu_0^2 r \sqrt{p^*}}{d_{\min}} \\
+ \frac{\mu_0^2 r \kappa^2 g^2}{\Delta^2 d_{\min} \sqrt{n}} (|\lambda_1| + \sqrt{np^*}) + \frac{R p^*}{\lambda_r^2 d_{\min}} (\sqrt{n}|\lambda_1| + n\sqrt{p^*}) + \frac{|\lambda_{r+1}|}{d_{\min}} \\
+ \norm{\bm{P}\bm{D}^{-1}}_{2,\infty} \sqrt{\frac{3(\gamma + 1) \log{n}}{d_{\min}}} + \frac{d_{\max}}{n d_{\min}}
\end{multline*}
as stated in the theorem. 

Furthermore, to see how this leads to \Cref{cor:strong_consistency_2}, suppose \Cref{assumption:low_rank} holds, and assume $p^* \succeq \max\left\{ \frac{\mu_0^2 r \log{n} }{n}, \frac{\mu_0^2 r^2}{n} \right\} $, and $\left| \lambda_1 / \lambda_r \right|$ is bounded. Applying \Cref{lemma:combined_succinct} to \eqref{eq:config_strong_condition_final} gives
\begin{equation*}\label{eq:config_strong_condition_final_lowrank}
h'(n) \succ \frac{ 1 }{ d_{\min} }\left( \mu_0 \sqrt{ r (\log{n} + r) p^* } + \mu_0^2 r \sqrt{ p^* } \right) + \norm{\bm{P}\bm{D}^{-1}}_{2,\infty} \sqrt{\frac{ \log{n} }{d_{\min}}}
\end{equation*}
as stated in  \Cref{cor:strong_consistency_2}.
\end{proof}

\begin{proof}[Proof of \Cref{prop:threshold_2}]
For any $i \in {\cal P}$, by \Cref{lemma:bound_periphery} and \Cref{lemma:combined_succinct}, and the boundedness of $\mu_0$ and $r$, \eqref{eq:config_bound_1} and \eqref{eq:config_bound_3} lead to
\begin{align*}
S_i &= \norm{\hat{\bm{P}}_{i,*} \hat{\bm{D}}^{-1} \bm{H}}_2 \\
&\le \norm{\bm{P}_{i,*}\bm{D}^{-1}\bm{H}}_2 + \norm{\hat{\bm{P}}_{i,*}\hat{\bm{D}}^{-1}\bm{H} - \bm{P}_{i,*}\bm{D}^{-1}\bm{H}}_2 \\
&\le \frac{d_{\max}}{(n-1) d_{\min}} + \norm{\hat{\bm{P}}_{i,*}\hat{\bm{D}}^{-1}\bm{H} - \bm{P}_{i,*}\bm{D}^{-1}\bm{H}}_2 \\
&\preceq \frac{ \sqrt{ p^* \log{n} }}{ d_{\min}} + \norm{\bm{P}\bm{D}^{-1}}_{2,\infty} \sqrt{\frac{\log{n}}{d_{\min}}};
\end{align*}
Similarly for any $i \in {\cal C}$, using the fact that $\norm{\bm{P}_{i,*}\bm{D}^{-1}\bm{H}} \ge h'(n)$ and \Cref{lemma:combined_succinct},
\begin{align*}
S_i &= \norm{\hat{\bm{P}}_{i,*}\hat{\bm{D}}^{-1}\bm{H}}_2 \\
&\ge h'(n) - \norm{\hat{\bm{P}}_{i,*}\hat{\bm{D}}^{-1}\bm{H} - \bm{P}_{i,*}\bm{D}^{-1}\bm{H}}_2 \\
&\succeq h'(n) - \frac{ \sqrt{ p^* \log{n} }}{ d_{\min}} - \norm{\bm{P}\bm{D}^{-1}}_{2,\infty} \sqrt{\frac{\log{n}}{d_{\min}}},
\end{align*}
for sufficiently large $n$ with probability $1-(B(r)+4)n^{-\gamma}$.

When $\min_{1 \le i,j \le n} \bm{P}_{ij} \simeq \max_{1 \le i,j \le n} \bm{P}_{ij} =p^*$, we have
\begin{align*}
\frac{ \sqrt{ p^* \log{n} }}{ d_{\min}} + \norm{\bm{P}\bm{D}^{-1}}_{2,\infty} \sqrt{\frac{\log{n}}{d_{\min}}} &\simeq \frac{ \sqrt{ p^* \log{n} }}{ np^* } + \norm{\bm{P}\bm{D}^{-1}}_{2,\infty} \sqrt{\frac{\log{n}}{ np^* }} \\
&\preceq \frac{ \sqrt{ \log{n} }}{ n\sqrt{p^*} } + \norm{\bm{P}}_{2,\infty}\norm{\bm{D}^{-1}}_2 \sqrt{\frac{\log{n}}{ np^* }} \\
&\simeq \frac{ \sqrt{ \log{n} }}{ n\sqrt{p^*} } + \frac{ \sqrt{n}p^* }{ np^* } \sqrt{\frac{\log{n}}{ np^* }} \\
&\simeq \frac{ \sqrt{ \log{n} }}{ n\sqrt{p^*} }.
\end{align*}

Furthermore, in the proof of \Cref{prop:threshold_1}, we have shown that $\hat{p} \simeq p^*$ with probability greater than $1-2n^{-\gamma}$. In this case, if $h'(n) \succ \frac{ \sqrt{ \log{n} }}{ n\sqrt{ {p^*}^{1+\epsilon} } }$, we have $ S_i \succ \frac{ \sqrt{ \log{n} }}{ n \sqrt{ \hat{p}^{1+\epsilon} } } $ for $i \in {\cal C}$, and $S_i \preceq \frac{ \sqrt{ \log{n} }}{ n \sqrt{ \hat{p} } }$ for $i \in {\cal P}$.

\end{proof}

\begin{proof}[Proof of \Cref{thm:weak_consistency_2}]
The key idea remains the same as in the proof of  \Cref{thm:weak_consistency_1}. We want to bound $\norm{ \bm{P} \bm{D}^{-1} \bm{H} - \hat{\bm{P}} \hat{\bm{D}}^{-1} \bm{H} }_{F}^2 $.
\begin{align*}
\norm{ \bm{P} \bm{D}^{-1} \bm{H} - \hat{\bm{P}} \hat{\bm{D}}^{-1} \bm{H} }_{F}^2 &\preceq \max\{ \rank(\bm{P} \bm{D}^{-1} \bm{H}), \rank(\hat{\bm{P}} \hat{\bm{D}}^{-1} \bm{H}) \} \cdot \norm{ \bm{P} \bm{D}^{-1} \bm{H} - \hat{\bm{P}} \hat{\bm{D}}^{-1} \bm{H} }_2^2 \\
&\le \max\{ \rank(\bm{P}), r \} \cdot \norm{ \bm{P} \bm{D}^{-1} - \hat{\bm{P}} \hat{\bm{D}}^{-1} }_2^2
\end{align*}
Using an argument similar to \eqref{eq:config_bound_2}, and \eqref{eq:config_bound_3}, with probability greater than $1-\frac{2}{n^\gamma}$ we have
\begin{displaymath}
\norm{ \bm{P} \bm{D}^{-1} - \hat{\bm{P}} \hat{\bm{D}}^{-1} }_2 \preceq \norm{ \bm{P} - \hat{\bm{P}} }_{2} \frac{1}{d_{\min}} + \norm{\bm{P}\bm{D}^{-1}}_{2} \sqrt{\frac{ \log{n} }{ d_{\min} }}.
\end{displaymath}
Meanwhile, by \Cref{thm:lei_rinaldo}, we also have
\begin{align*}
\norm{ \bm{P} - \hat{\bm{P}} }_{2} &= \norm{ \bm{P} - \bm{A} + \bm{A} - \hat{\bm{P}} }_{2} \\
&\le \norm{ \bm{P} - \bm{A} }_2 + \norm{ \bm{A} - \hat{\bm{P}} }_2 \\
&\le \norm{ \bm{P} - \bm{A} }_2 + |\hat{\lambda}_{r+1}| \\
&\le 2\norm{ \bm{P} - \bm{A} }_2 + | \lambda_{r+1} | \\
&\preceq \sqrt{np^*} + | \lambda_{r+1} |,
\end{align*}
with probability at least $1-\frac{1}{n^\gamma}$. Therefore, combining the above equations, we get
\begin{displaymath}
\norm{ \bm{P} \bm{D}^{-1} \bm{H} - \hat{\bm{P}} \hat{\bm{D}}^{-1} \bm{H} }_{F}^2 \preceq \max\{ \rank(\bm{P}), r \} \cdot \left[ \frac{ np^* + \lambda_{r+1}^2 }{ d_{\min}^2 } + \norm{\bm{P}\bm{D}^{-1}}_2^2 \frac{ \log{n} }{ d_{\min} } \right],
\end{displaymath}
with probability at least $1-\frac{3}{n^\gamma}$, and the number of misclassified nodes satisfies
$$
M' \preceq \max\{ \rank(\bm{P}), r \} \cdot \frac{ \left[ np^* + \lambda_{r+1}^2 + \norm{\bm{P}\bm{D}^{-1}}_2^2 \cdot d_{\min} \cdot \log{n} \right]}{ d_{\min}^2 \left[h'(n) - \frac{d_{\max}}{(n-1)d_{\min}} \right]^2 }.
$$

\end{proof}

\section{ Additional Simulation Results }\label{appendix:additional_simulation}

In this section, we include the additional simulation results, where the core size and the periphery size are different. We can see that our method achieves the best performance across different settings, which is consistent with the balanced cases. 

\begin{figure}
\centering
\begin{subfigure}{0.96\textwidth}
  \centering
  \includegraphics[width=1\textwidth]{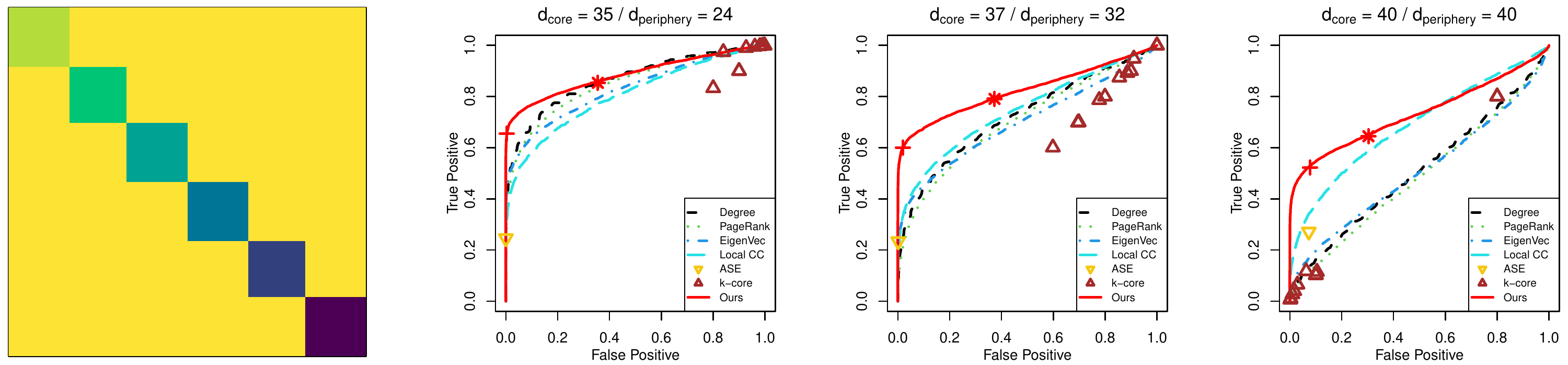}
  \caption{Graphon 1}
\end{subfigure} \\
\begin{subfigure}{0.96\textwidth}
  \centering
  \includegraphics[width=1\textwidth]{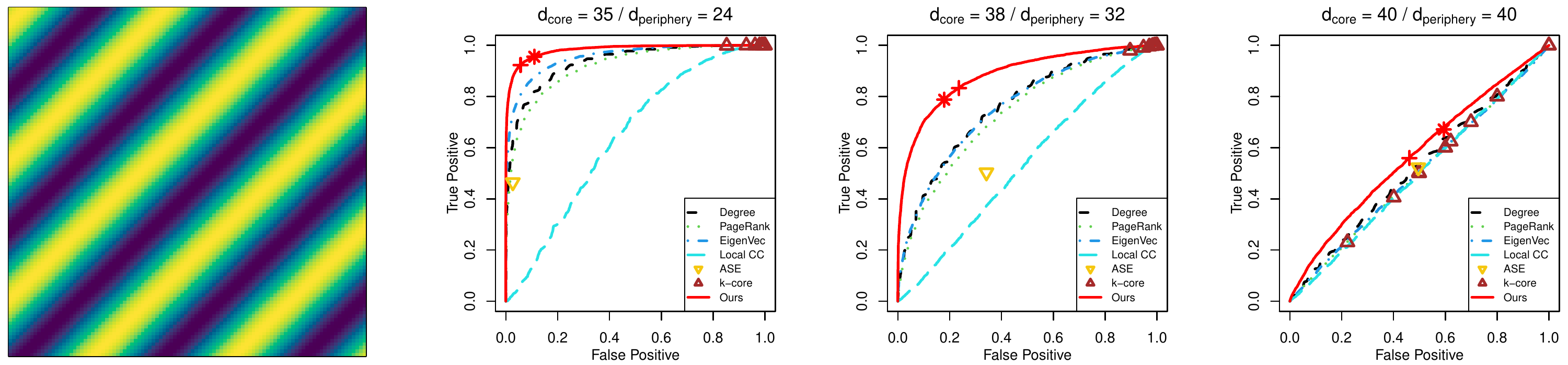}
  \caption{Graphon 2}
\end{subfigure} \\
\begin{subfigure}{0.96\textwidth}
  \centering
  \includegraphics[width=1\textwidth]{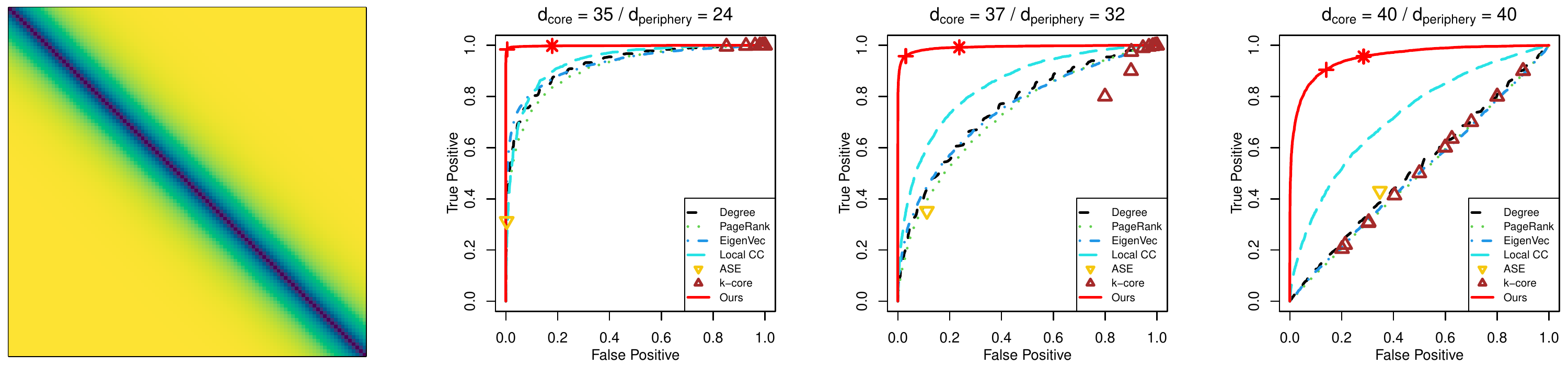}
  \caption{Graphon 3}
\end{subfigure} \\
\caption{ Erd\"{o}s-Renyi periphery. $N_{\cal C}=700$, $N_{\cal P}=1300$. }
\end{figure}

\begin{figure}
\centering
\begin{subfigure}{0.96\textwidth}
  \centering
  \includegraphics[width=1\textwidth]{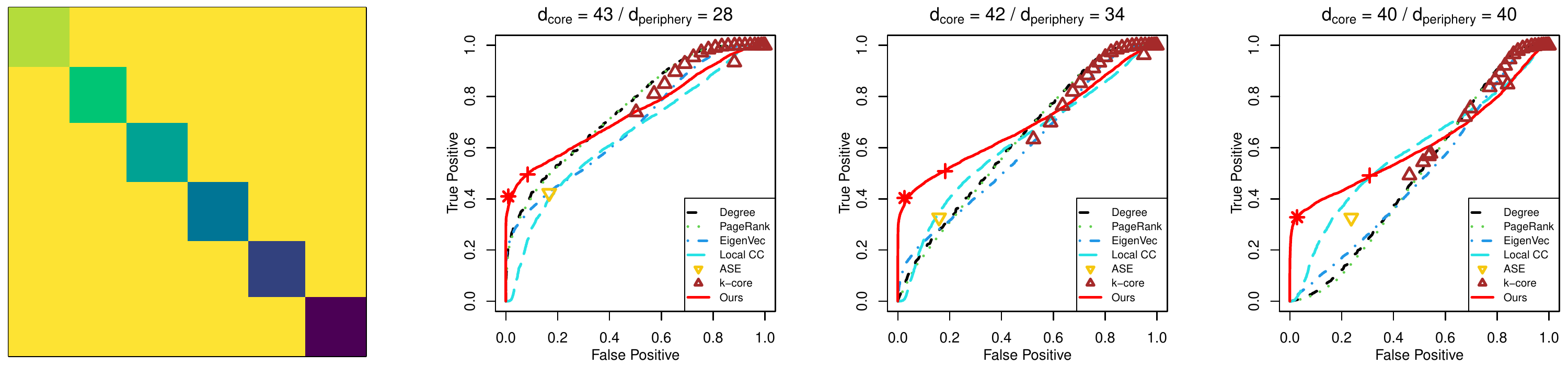}
  \caption{Graphon 1}
\end{subfigure} \\
\begin{subfigure}{0.96\textwidth}
  \centering
  \includegraphics[width=1\textwidth]{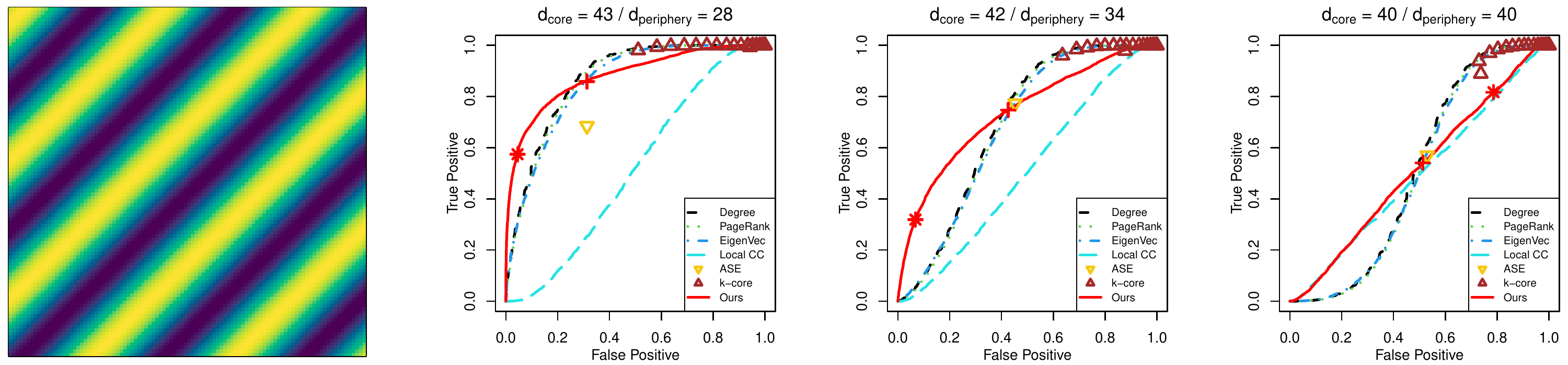}
  \caption{Graphon 2}
\end{subfigure} \\
\begin{subfigure}{0.96\textwidth}
  \centering
  \includegraphics[width=1\textwidth]{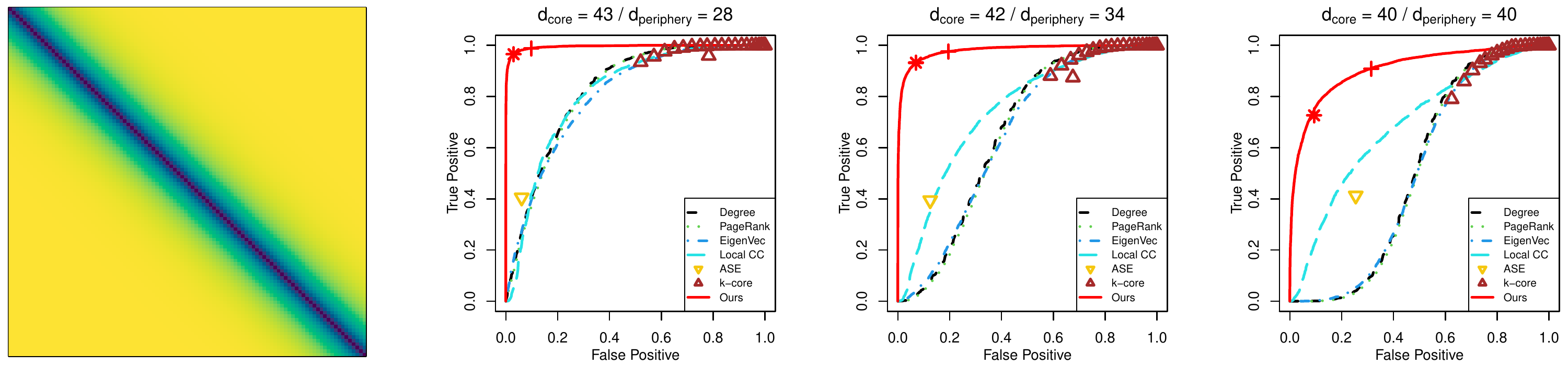}
  \caption{Graphon 3}
\end{subfigure} \\
\caption{ Configuration periphery. $N_{\cal C}=700$, $N_{\cal P}=1300$. }
\end{figure}

\begin{figure}
\centering
\begin{subfigure}{0.96\textwidth}
  \centering
  \includegraphics[width=1\textwidth]{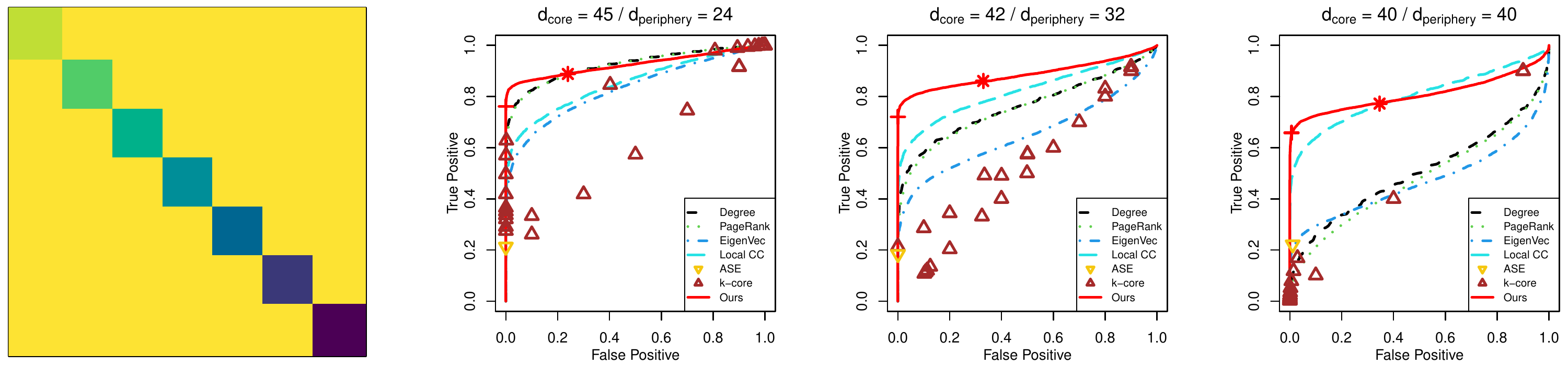}
  \caption{Graphon 1}
\end{subfigure} \\
\begin{subfigure}{0.96\textwidth}
  \centering
  \includegraphics[width=1\textwidth]{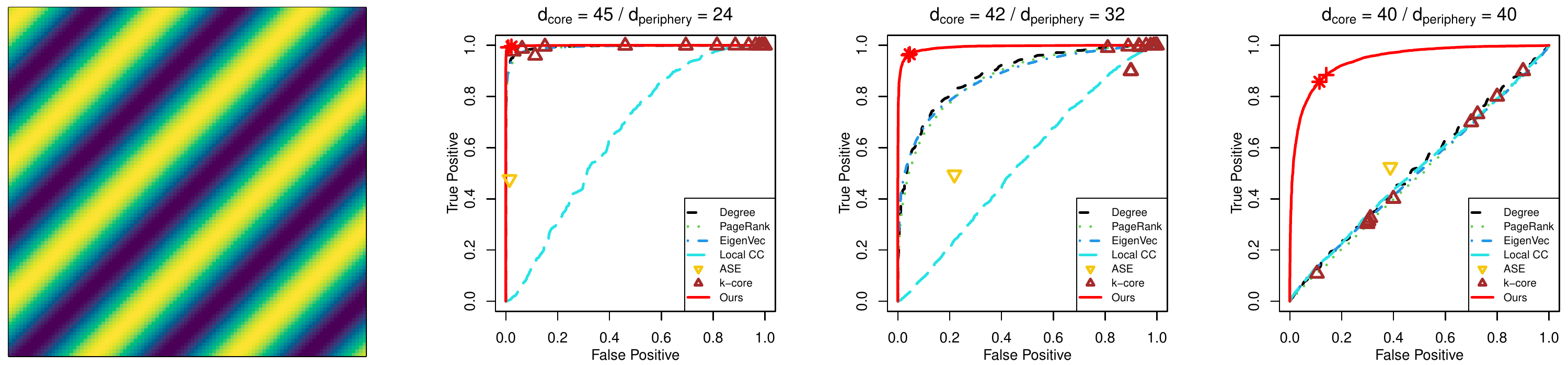}
  \caption{Graphon 2}
\end{subfigure} \\
\begin{subfigure}{0.96\textwidth}
  \centering
  \includegraphics[width=1\textwidth]{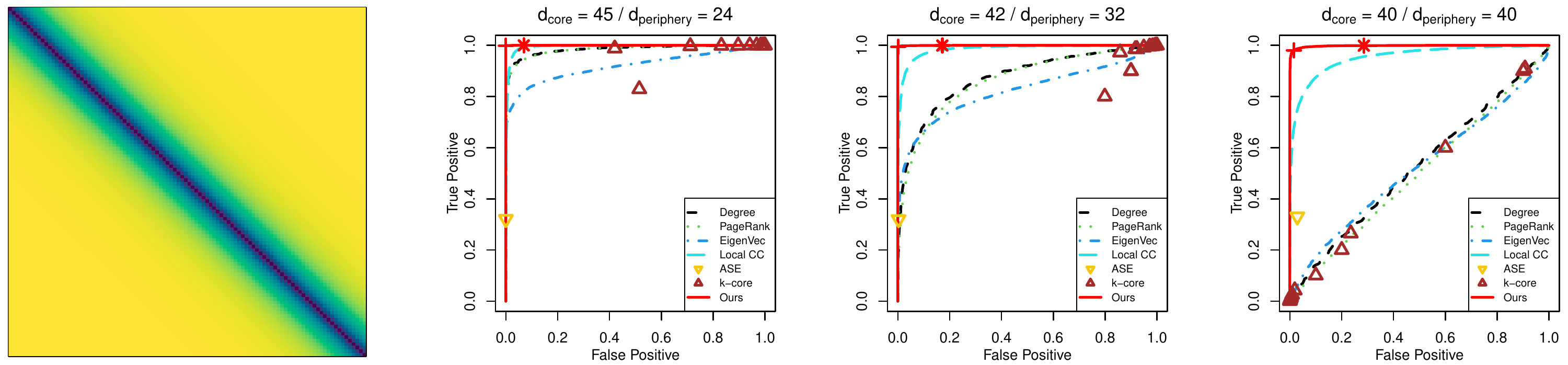}
  \caption{Graphon 3}
\end{subfigure} \\
\caption{ Erd\"{o}s-Renyi periphery. $N_{\cal C}=1300$, $N_{\cal P}=700$. }
\end{figure}

\begin{figure}
\centering
\begin{subfigure}{0.96\textwidth}
  \centering
  \includegraphics[width=1\textwidth]{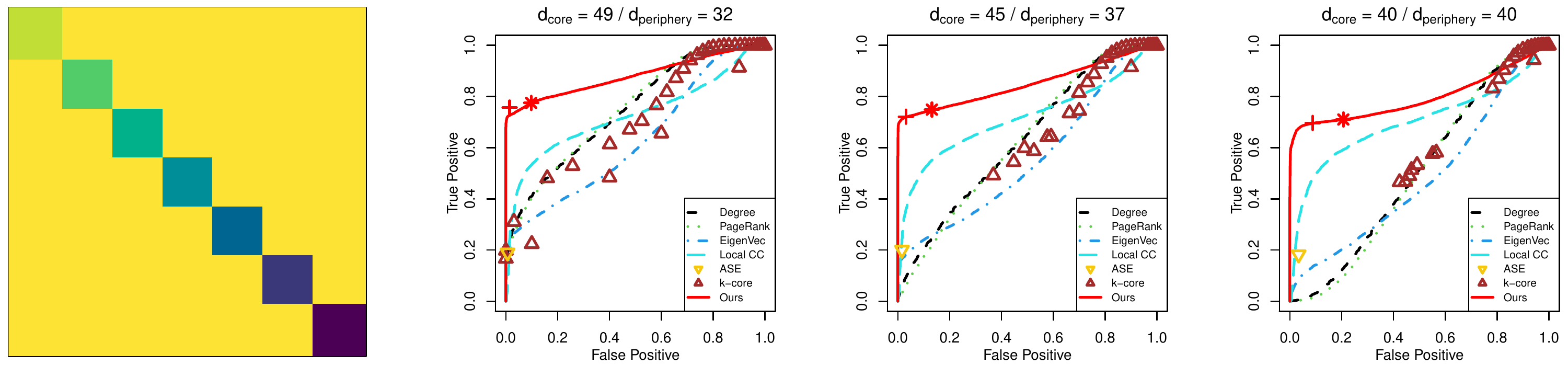}
  \caption{Graphon 1}
\end{subfigure} \\
\begin{subfigure}{0.96\textwidth}
  \centering
  \includegraphics[width=1\textwidth]{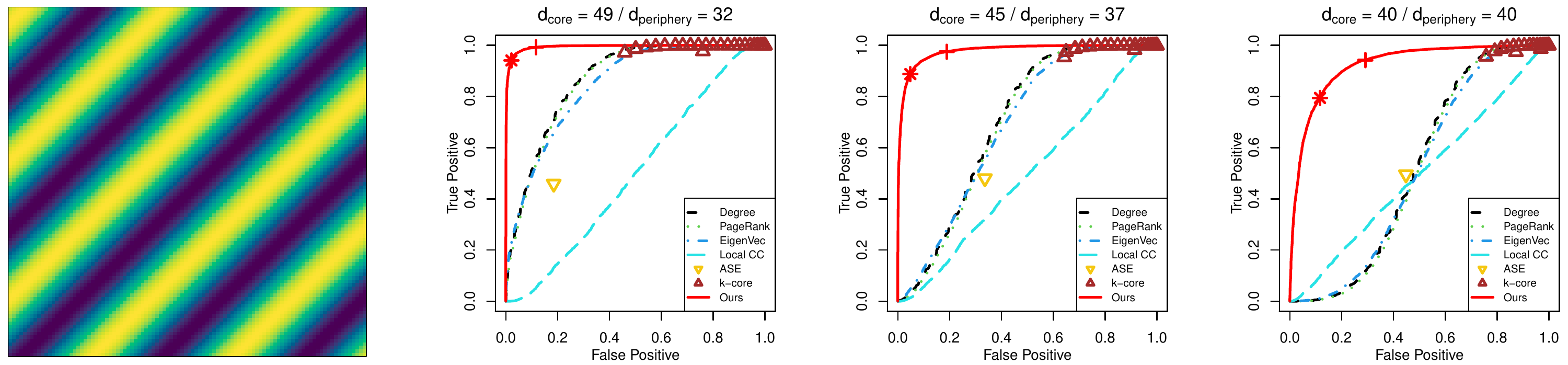}
  \caption{Graphon 2}
\end{subfigure} \\
\begin{subfigure}{0.96\textwidth}
  \centering
  \includegraphics[width=1\textwidth]{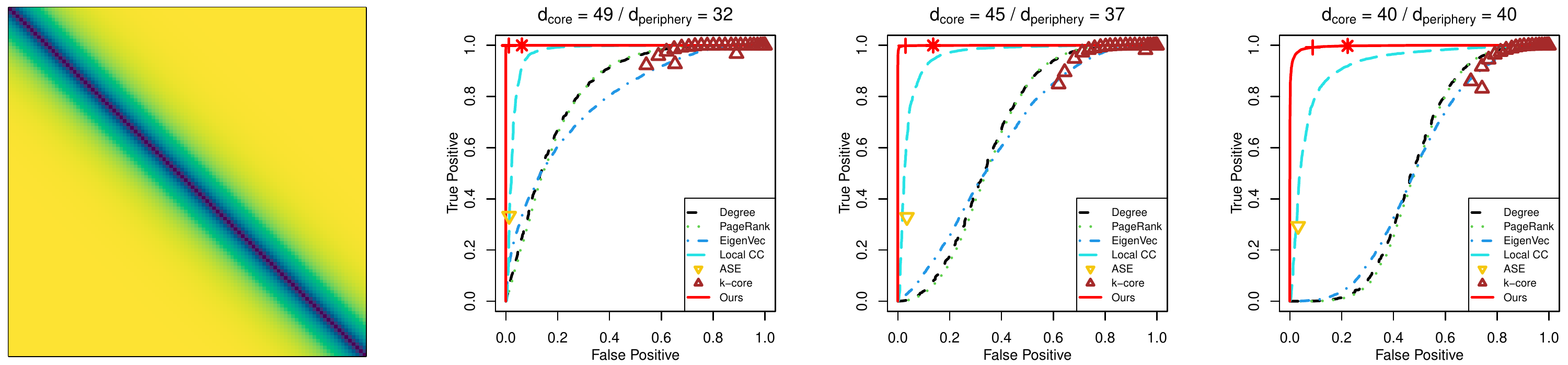}
  \caption{Graphon 3}
\end{subfigure} \\
\caption{ Configuration periphery. $N_{\cal C}=1300$, $N_{\cal P}=700$. }
\end{figure}

\end{appendix}